\documentclass{article}

\PassOptionsToPackage{numbers, compress}{natbib}

\usepackage[utf8]{inputenc} % allow utf-8 input
\usepackage[T1]{fontenc}    % use 8-bit T1 fonts
\usepackage{hyperref}       % hyperlinks
\usepackage{url}            % simple URL typesetting
\usepackage{booktabs}       % professional-quality tables
\usepackage{amsfonts}       % blackboard math symbols
\usepackage{nicefrac}       % compact symbols for 1/2, etc.
\usepackage{microtype}      % microtypography
\usepackage{xcolor}         % colors
\usepackage{tcolorbox}
\usepackage[HTML]{xcolor}
\usepackage{booktabs}
\usepackage{tabularx}
\usepackage{array}
\usepackage{ragged2e}
\usepackage{xurl}
\usepackage[bottom]{footmisc}
\usepackage{threeparttable}
\usepackage[numbers,sort&compress]{natbib}
\usepackage{parskip}
\usepackage[margin=1in]{geometry}

\newcolumntype{L}[1]{>{\RaggedRight\arraybackslash}p{#1}}
\newcolumntype{Y}{>{\RaggedRight\arraybackslash}X}

\newcommand{\assetid}[1]{{\scriptsize\path{#1}}}

\newcommand{\assetlink}[2]{\href{#1}{\texttt{#2}}}
\newcommand{\rowrule}{\addlinespace[0.25em]\specialrule{0.15pt}{0pt}{0pt}\addlinespace[0.25em]}

\PassOptionsToPackage{square,numbers,sort&compress}{natbib}

\usepackage[utf8]{inputenc} % allow utf-8 input
\usepackage[T1]{fontenc}    % use 8-bit T1 fonts
\usepackage{hyperref}       % hyperlinks
\usepackage{url}            % simple URL typesetting
\usepackage{booktabs}       % professional-quality tables
\usepackage{amsfonts}       % blackboard math symbols
\usepackage{nicefrac}       % compact symbols for 1/2, etc.
\usepackage{microtype}      % microtypography
\usepackage{xcolor}         % colors
\usepackage{pifont}
\usepackage{enumitem}

\usepackage{graphicx}
\usepackage{algorithm}
\usepackage{dsfont}
\usepackage{amssymb}
\usepackage{algpseudocode}
\usepackage{subcaption}
\usepackage{xcolor}

\usepackage{multirow} 
\graphicspath{{figs/}}
\usepackage[toc,page]{appendix}
\usepackage{tabularx}
\usepackage{relsize}
\usepackage{booktabs} % Required for \toprule, \midrule, and \bottomrule
\usepackage{geometry} % Optional, for page size and margin setup
\usepackage{amsmath}
\usepackage{amsthm}
\usepackage{bbm}
\usepackage{wrapfig,lipsum,booktabs}

\definecolor{boxbg}{HTML}{fff8e8}

\DeclareMathOperator*{\argmax}{arg\,max}

\newcommand{\piref}{\pi_{\text{ref}}}

\newcommand{\logit}{\textnormal{logit}}

\newtheorem{theorem}{Theorem}[section]
\newtheorem{proposition}[theorem]{Proposition}
\newtheorem{lemma}[theorem]{Lemma}
\newtheorem{corollary}[theorem]{Corollary}

\usepackage{tikz}
\usetikzlibrary{positioning, shapes, arrows, calc}

\title{Truncate Bad, Upweight Good: \\ BoN-Style Distillation via Rank-Based Classification}

\author{%
  Yarin Bar\textsuperscript{1} \quad
  Yaniv Romano\textsuperscript{1,2} \\[0.5em]
  \textsuperscript{1}Department of Electrical and Computer Engineering\\[0.2em]
  \textsuperscript{2}Department of Computer Science\\[0.3em]
  Technion--Israel Institute of Technology\\[0.5em]
  \texttt{yarinbar@campus.technion.ac.il}
  \quad
  \texttt{yromano@technion.ac.il}
}

\date{}

\begin{document}

\maketitle

\begin{abstract}
Inference-time selection methods, such as Best-of-N, improve generation by sampling a pool of candidates and selecting the top-ranked completion according to a reward model. Distillation seeks to amortize this procedure into a single policy by replacing raw rewards with in-pool ranks and learning a policy that upweights higher-ranked completions. However, existing rank-based policies typically use smooth full-support reweighting, so low-ranked completions receive less mass but remain in the target support. Although a sharper reweighting reduces lower-tail mass, it also increases reliance on brittle ranking at the top made by a single reward model. We propose \textbf{TUP}: a \textbf{T}runcate-bad, \textbf{U}pweight-good \textbf{P}olicy that removes low-ranked completions from the support and reweights only the retained upper tail with a tunable sharpness. TUP admits a closed-form, prompt-independent normalization and can be trained fully offline via binary cross-entropy, using shifted-truncated win-rates as soft labels and distilled-to-reference log-likelihood ratios as logits. 
Theoretically, under certain assumptions, we show that for any unknown oracle reward, the best monotone rank-reweighting can be matched by a lower-tail truncation rule, providing formal support for removing the lower tail rather than merely downweighting it. 
Empirically, we show that TUP is competitive with strong offline alignment baselines.

\end{abstract}

\section{Introduction}
\label{sec:intro}

Inference-time selection methods such as best-of-$N$ (BoN) have emerged as an effective tool for improving language models' generation quality~\citep{huang_is_2025,beirami_theoretical_2025}. BoN samples several completions from a base policy, scores them using a reward model, and returns the highest-scoring one~\citep{ivison_unpacking_2024, ankner_critique-out-loud_2024, song_good_2025}. The price is computational, since each prompt requires multiple generated completions and reward evaluations~\citep{brown2025large}. This has motivated a growing line of work on BoN-style distillation~\citep{sessa_bond_2025, matrenok_quantile_2025, balashankar_infalign_nodate, gui_bonbon_2024, dong_raft_2023, beirami_theoretical_2025}, where rather than carrying out selection during deployment, one trains a single policy to internalize the behavior of the inference-time selector.

A key insight from recent BoN-style distillation work is that the relevant training signal for a completion is not its raw reward in isolation, but its relative rank within the sampled pool. This relative rank is often referred to as the completion \emph{win-rate}. In turn, existing distillation methods aim to maximize the win-rate while penalizing the KL distance from the base policy~\citep{matrenok_quantile_2025, balashankar_infalign_nodate, gui_bonbon_2024}. This results in policies that apply smooth full-support reweighting, favoring higher-ranked completions while keeping every completion in the target support. Sharper reweighting reduces the probability mass of clearly low-ranked completions, but it also concentrates more mass on the very top of the reward-model ranking. As a result, over-sharpening can move the policy further from the reference model (smaller KL-penalty), increasing the risk of reward hacking~\citep{michaud2020understanding,tien2023causal, jinnai2024regularized, rafailov_2024_scaling_laws}. Since reward models are imperfect proxies for unknown oracle or human preferences~\citep{frick_how_nodate, mac_kim_rethinking_2025, huang_is_2025, khalaf_inference-time_2026, nakano2021webgpt, zheng2023judging, pmlr-v202-gao23h, rafailov_2024_scaling_laws}, recent work shows that coarse preferences tend to persist across different evaluators, while fine-grained distinctions among top-ranked completions are less consistently agreed upon~\citep{lambert_rewardbench_2025, zhang_chasing_2026, landesberg_when_2026}.

Consequently, we argue that the design of a BoN-style distillation policy should separate two choices: (i) how much of the lower-ranked tail should be removed from the support, and (ii) how sharply the retained higher-ranked completions should be upweighted. To this end, we propose \textbf{TUP}, a \textbf{T}runcate-bad, \textbf{U}pweight-good \textbf{P}olicy for BoN-style distillation. TUP decouples these choices with a shifted-truncated win-rate transform. 
Completions whose win-rate falls below the threshold are assigned zero mass, while those above it are softly reweighted by their relative ranking. The threshold parameter controls lower-tail truncation, whereas the upper-tail sharpness parameter controls the upweighting strength within the retained set. By introducing separate parameters for lower-tail truncation and upper-tail sharpness, this design preserves the BoN principle of favoring higher-ranked completions while (i) preventing probability mass from being assigned to clearly undesirable completions, and (ii) avoiding over-sharpening due to possible ranking mistakes at the top.

The resulting target policy remains simple to train from offline data. Each prompt is paired with a pool of completions sampled from the reference policy and scored by a reward model. From these scores, we compute empirical in-pool ranks and convert them into shifted-truncated win-rate labels; by subtracting the truncation parameter, we assign the lower-tail completions an exact zero label. Pairing with the logit transformation, we then train the distilled language model with binary cross-entropy (BCE) loss.

\paragraph{Contributions.}
In this work, we make three key contributions.
\begin{enumerate}
\item \textbf{A truncate-bad, upweight-good policy for BoN-style distillation.}
We propose a rank-based target policy that removes low-ranked completions using a win-rate threshold and softly reweights the retained upper tail.

\item \textbf{Theoretical support for lower-tail truncation and within-tail upweighting.}
Under certain assumptions, we give two theoretical justifications for TUP, using the oracle win-rate as the performance criterion. First, we show that the best hard lower-tail truncation rule can match the best monotone policy that reweights completions by proxy-reward win-rate. Second, after fixing a truncation threshold, we show that smooth upweighting of retained completions can further improve the oracle win-rate when, informally, higher proxy-reward win-rates tend to correspond to higher oracle win-rates within the retained tail.

\item \textbf{Benchmark results and ablations.}
We evaluate TUP on the QRPO benchmark training Llama-8B Tülu 3 SFT on UltraFeedbak and Magpie Air and Mistral-7B-Instruct-v0.2 on Magpie Air, comparing to four leading offline-alignment baseline methods. We evaluate performance using three different reward models, both on held-out data and AlpacaEval~\citep{dubois2025lengthcontrolledalpacaevalsimpleway}. Full implementation of TUP and the experiments are available at \href{https://github.com/yarinbar/truncate-bad-upweight-good}{github.com/yarinbar/truncate-bad-upweight-good}.

\end{enumerate}

\section{Preliminaries and related works}
\label{sec:prelim}

\begin{figure}[tp] % [htbp] suggests placement: here, top, bottom, or page
    \centering
    \includegraphics[width=0.94\textwidth]{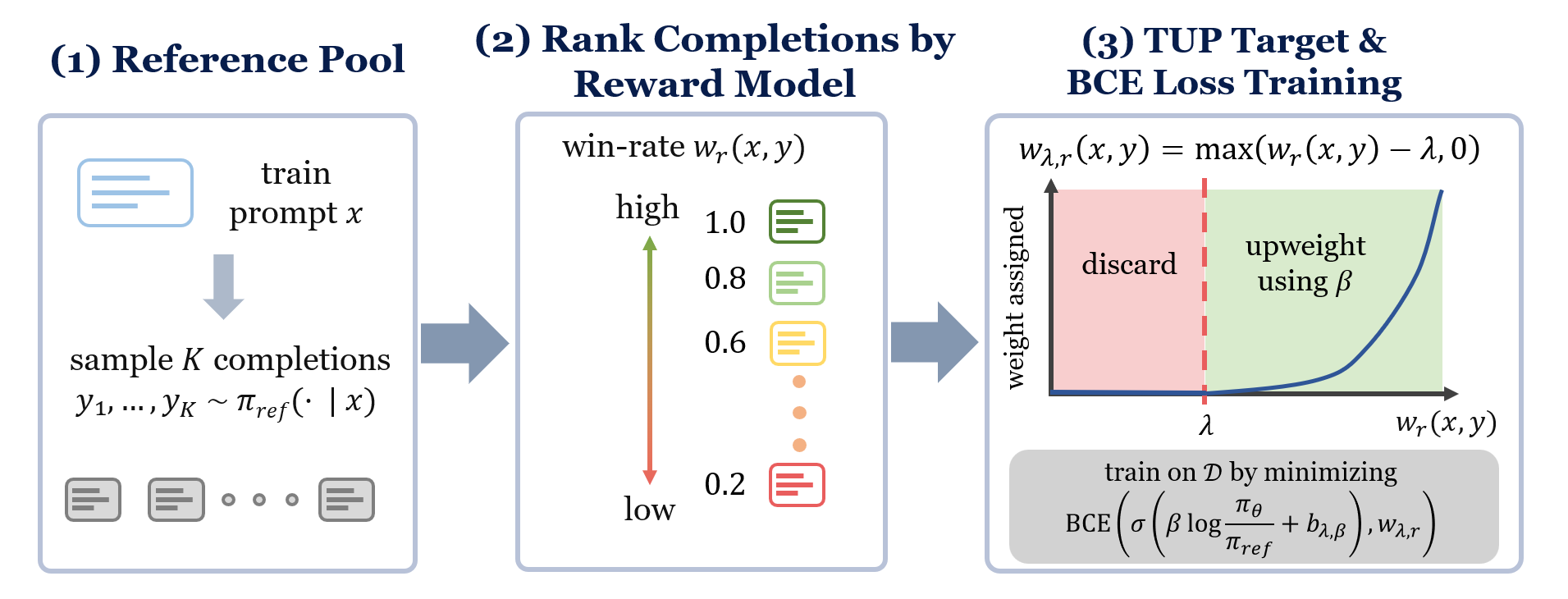} 
    \caption{Illustration of TUP.}
    \label{fig:my_label}
\end{figure}

We denote a language model as a policy $\pi(y \mid x)$ mapping a prompt $x \in \mathcal{X}$ to a distribution over completions $y \in \mathcal{Y}$. Prompts are drawn from a distribution $\mathcal D$ over $\mathcal X$. Let $\pi_{\text{ref}}(y \mid x)$ denote the reference policy, e.g., obtained via supervised fine-tuning, and let $r(x,y) \in \mathbb{R}$ be a reward function.

The standard post-training alignment objective seeks a policy that maximizes expected reward while remaining close to the reference policy:
\begin{equation}
\label{eq:kl_objective}
\pi_{r,\beta}^{\star}(\cdot|x)
=
\arg\max_{\pi}
\mathbb{E}_{y\sim\pi(\cdot|x)}[r(x,y)]
-
\beta D_{\text{KL}}(\pi(\cdot|x)\,\|\,\pi_{\text{ref}}(\cdot|x)),
\end{equation}
where $\beta > 0$ controls the reward--KL tradeoff. This KL-regularized reward-maximization objective is a standard formulation in RLHF for language-model alignment~\citep{christiano_deep_2017, stiennon_rlhf, ziegler2020finetuninglanguagemodelshuman, ouyang_training_2022, bai2022traininghelpfulharmlessassistant}, and is commonly optimized with online reinforcement-learning methods such as PPO~\citep{schulman_proximal_2017} and GRPO~\citep{shao_deepseekmath_2024}.

An alternative paradigm, which we follow in this paper, builds directly on the closed-form solution to~\eqref{eq:kl_objective}, known as the Gibbs policy:
\begin{equation}
\label{eq:gibbs_policy}
\pi_{r,\beta}^{\star}(y|x)
=
\frac{1}{Z(x)} \pi_{\text{ref}}(y|x)
\exp\left(\frac{r(x,y)}{\beta}\right), \ \ \ Z(x)
=
\mathbb{E}_{Y\sim\pi_{\text{ref}}(\cdot|x)}
\left[\exp\left(\frac{r(x,Y)}{\beta}\right)\right],
\end{equation}
where $Z(x)$ is the partition function. Rearranging~\eqref{eq:gibbs_policy} yields
\begin{equation}
\label{eq:reward_regression}
r(x,y)
=
\beta \log Z(x)
+
\beta \log \frac{\pi_{r,\beta}^{\star}(y|x)}{\pi_{\text{ref}}(y|x)}.
\end{equation}
This relation suggests a direct fitting strategy. Replacing the unknown optimal policy $\pi_{r,\beta}^{\star}$ with a parameterized model $\pi_\theta$ would allow one to train the policy by regressing the right-hand side of~\eqref{eq:reward_regression} to predict the reward. The difficulty is that the partition function $Z(x)$ is prompt-dependent and impractical to compute. Preference-optimization methods avoid the need to compute  $Z(x)$ in different but related ways. DPO~\citep{rafailov_direct_2023}, IPO~\citep{garg_ipo_2025}, and SimPO~\citep{meng_simpo_2024} learn from pairwise or binary preference signals, while REBEL~\citep{gao_rebel_2024} regresses relative reward differences between completions.

Although standard post-training alignment typically treats the raw reward $r(x,y)$ as the quantity to optimize, this is not the most natural choice for inference-time selection algorithms such as BoN. Recent inference-aware and rank-based methods, including InfAlign~\citep{balashankar_infalign_nodate} and QRPO~\citep{matrenok_quantile_2025}, motivate replacing raw reward values with relative-rank quantities. A natural choice is the population \emph{win-rate}
\begin{equation}
\label{eq:population_winrate}
w_{r}(x,y)
:=
\mathbb{P}_{Y^{\prime}\sim\pi_{\text{ref}}(\cdot|x)}
\bigl(r(x,Y^{\prime}) \le r(x,y)\bigr),
\end{equation}
which measures how often $y$ beats an independent reference sample. In practice, one can estimate the win-rate from a finite pool of completions.

This win-rate parametrization is useful not only because it matches the relative nature of BoN-style selection, but also because it removes the prompt-dependent partition function in~\eqref{eq:reward_regression}. When $Y \sim \pi_{\mathrm{ref}}(\cdot\mid x)$ and the conditional reward distribution is continuous, the probability integral transform implies that $w_r(x,Y)\sim U[0,1]$. Therefore, for any nonnegative reweighting function $g:[0,1]\to[0,\infty)$ with finite integral, the rank-based policy
\begin{equation}
    \label{eq:pi_g}
    \pi_g(y\mid x)
    =
    \frac{1}{Z_g} g(w_r(x,y)) \pi_{\mathrm{ref}}(y\mid x),
    \qquad
    Z_g
    =
    \int_0^1 g(u)\,du,
\end{equation}
has a normalizer $Z_g$ independent of the prompt $x$. Many rank-based policies follow this construction. For example, classic Best-of-$N$ induces $g_{\mathrm{BoN},N}(w)\propto w^{N-1}$~\citep{gui_bonbon_2024, beirami_theoretical_2025}, and the standard QRPO formulation induces $g_{\mathrm{QRPO},\beta}(w)\propto e^{w/\beta}$~\citep{matrenok_quantile_2025}. Both assign positive weight over the entire support, leaving open how the win-rate should be transformed into a target-policy reweighting.

This full-support property clarifies how our target differs from related alignment methods. Iterative BoN-distillation methods such as BOND~\citep{sessa_bond_2025} and Faster-WIND~\citep{yang_faster_2025} also aim to approximate or distill Best-of-$N$-type behavior, but they inherit the smooth reweighting view rather than defining a hard lower-tail removal rule. Other objectives, including RAFT~\citep{dong_raft_2023}, SLiC-HF~\citep{zhao_slic-hf_2023}, and RRHF~\citep{yuan_rrhf_2023}, use reward filtering or relative preference rankings to construct training signals, but they do not specify a prompt-independent rank-transform density that first truncates the lower tail and then softly reweights the retained upper tail. Our work makes this truncation-and-reweighting operation explicit, yielding a finite prompt-independent normalization and a target that can be fit as a soft-classification problem using the transformed win-rate as labels.

\section{Method}
\label{sec:method}
% !TeX root = ../neurips_2026.tex

\begin{figure}[tp]

     \centering
     % First Subfigure
     \begin{subfigure}[t]{0.41\textwidth}
         \centering
         \includegraphics[width=\textwidth]{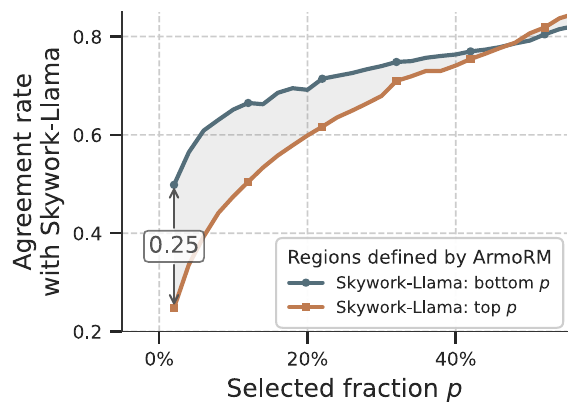}
     \end{subfigure}
     \hfill 
     % Second Subfigure
     \begin{subfigure}[t]{0.50\textwidth}
         \centering
         \includegraphics[width=\textwidth]{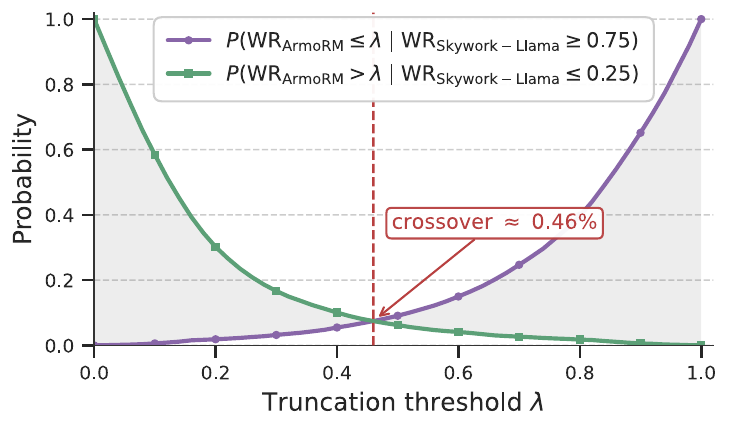}
     \end{subfigure}
     \caption{Reward model agreement \textbf{(left)} and the cost--benefit tradeoff of $\lambda$-truncation \textbf{(right)} on UltraFeedback reference completions. We compare ArmoRM, used as the proxy reward model, against the independent Skywork-v2-Llama reward model. \textbf{Left:} Fraction of completions placed by ArmoRM in the top or bottom $p$ region that Skywork also places in the same region, as a function of $p$. \textbf{Right:} Probability that $\lambda$-truncation discards a Skywork-top-$25\%$ completion or retains a Skywork-bottom-$25\%$ completion. The crossover marks where these error probabilities are equal.}
    \label{fig:top_bot_disagreement}
\end{figure}

\subsection{The proposed truncate-bad, upweight-good target policy}
\label{sec:our-policy}

Gibbs-like policies \eqref{eq:pi_g} that reweight $\pi_{\mathrm{ref}}$ by a smooth monotone function of the proxy win-rate, such as QRPO~\citep{matrenok_quantile_2025}, face the sharpness tradeoff we discussed above. 

Figure~\ref{fig:top_bot_disagreement} (left) provides an empirical illustration motivating our choice to decouple lower-tail truncation from upper-tail sharpness. When the same reference completions are ranked by ArmoRM~\citep{wang_interpretable_2024} and by an independent Skywork-v2-Llama reward~\citep{liu_skywork-reward-v2_2025}, the models tend to agree more on the bottom of the pool than on the top. Thus, removing the lower tail targets a comparatively stable region across reward models, whereas sharpening toward the top places disproportionate trust in fine-grained top-rank distinctions with weaker cross-model agreement. Additional such comparisons are available in Appendix~\ref{appx:additional_exps}.

We instantiate this principle by introducing a truncation level \(\lambda\in(0,1)\), 
which determines whether a completion remains in the target support, and a sharpness 
parameter \(\beta>0\), which controls how strongly the retained completions are 
reweighted. Given \(\lambda\), we define the \emph{shifted-truncated win-rate}
\begin{equation}
    w_{\lambda, r}(x,y) \;:=\; \max\!\bigl(w_r(x,y) - \lambda,\, 0\bigr),
\end{equation}
and use it to define the transformed reward
\begin{equation}
    R_\lambda(x,y) \;:=\; \mathrm{logit}\!\bigl(w_{\lambda, r}(x,y)\bigr) 
    \;=\; \log\!\frac{w_{\lambda, r}(x,y)}{1 - w_{\lambda, r}(x,y)}.
\end{equation}
The truncation level \(\lambda\) acts as a quality threshold. Completions with 
\(w_r(x,y)\le \lambda\) are mapped to \(R_\lambda(x,y)=-\infty\), while completions 
above the threshold are retained and scored by the log-odds of their shifted win-rate. 
Thus, \(\lambda\) imposes a hard quality floor, whereas the transformed reward preserves 
a smooth ordering among the retained completions. As shown in 
Figure~\ref{fig:top_bot_disagreement} (right), increasing \(\lambda\) lowers the 
probability of retaining completions that the independent reward model rank poorly, but 
also raises the probability of discarding completions this model rank highly.

We now derive the KL-regularized target policy induced by the transformed reward $R_\lambda(x,y)$. The key point is 
that truncation preserves the tractability of win-rate-based objectives. When the reward 
distribution is continuous under \(Y\sim\pi_{\mathrm{ref}}(\cdot\mid x)\), the population 
win-rate \(w_r(x,Y)\) is uniform on \([0,1]\). After shifting and truncating this variable, 
the Gibbs normalizer remains independent of the prompt and can be computed in closed form.

\begin{proposition}[Closed-form target under the truncated reward]
\label{prop:optimal_policy_lambda}
Assume that, for the fixed prompt \(x\), the distribution of \(r(x,Y)\) under 
\(Y\sim\pi_{\mathrm{ref}}(\cdot\mid x)\) is continuous. Then, for every truncation level 
\(\lambda \in (0,1)\) and regularization parameter \(\beta > 0\), the KL-regularized 
objective in~\eqref{eq:kl_objective} with reward \(R_\lambda\) admits the unique Gibbs 
solution
\begin{equation}
    \label{eq:optimal_policy_lambda}
    \pi_{\lambda,\beta}^\star(y \mid x)
    \;=\;
    \frac{1}{Z_{\lambda,\beta}}\,
    \pi_{\mathrm{ref}}(y \mid x)
    \left(\frac{w_{\lambda, r}(x,y)}{1 - w_{\lambda, r}(x,y)}\right)^{1/\beta},
\end{equation}
where the normalization constant is
\begin{equation}
    \label{eq:Z_lambda_beta}
    Z_{\lambda,\beta}
    \;=\;
    \int_0^{1-\lambda}
    t^{1/\beta}(1-t)^{-1/\beta}\,dt
    \;=\;
    \mathrm{Beta}_{1-\lambda}\!\bigl(1+\tfrac{1}{\beta},\, 1-\tfrac{1}{\beta}\bigr).
\end{equation}
Above, $\mathrm{Beta}_{u}(a,b)$ denotes the incomplete Beta function.
\end{proposition}

The proposition follows by substituting \(R_\lambda\) into the Gibbs solution in~\eqref{eq:gibbs_policy}; Appendix~\ref{app:normalizer_incomplete_beta} gives the population normalizer calculation. The result separates support selection from within-support reweighting, where \(\lambda\) controls which completions remain in the support and the sharpening parameter \(\beta\) controls how strongly the target policy departs from \(\pi_{\mathrm{ref}}\) within the retained set.

\begin{figure}[tp]
     \centering
     % First Subfigure
     \begin{subfigure}[t]{0.70\textwidth}
         \centering
         \includegraphics[width=\textwidth]{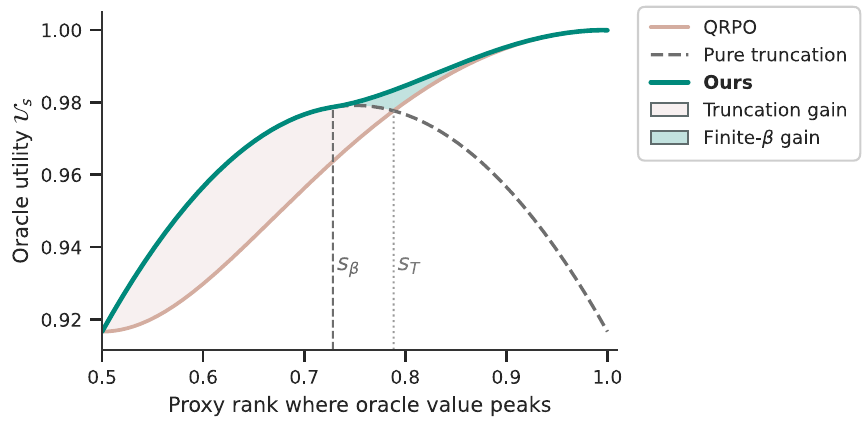}
     \end{subfigure}
    \caption{Fixed-threshold illustration for the quadratic un-normalized oracle-proxy profile $m_s(w)=1-(w-s)^2$ at $\lambda=1/2$. The x axis is the proxy-rank location $s$ where oracle utility is maximized, and the y axis shows the unnormalized oracle utility $\mathcal U_s$. The curves compare the best QRPO policy, pure truncation, and the best TUP reweighting at the fixed threshold $\sup_{\beta\in(0,\infty]}\mathcal U_s(1/2,\beta)$. The vertical dashed line marks $s_\beta$, above which finite $\beta$ improves over pure truncation, and the dotted line marks $s_T$, where pure truncation no longer outperforms the best QRPO curve.}

    \label{fig:theory}
\end{figure}

\subsection{Theoretical merits of truncate-bad, upweight-good policies}
\label{sec:oracle_win-rate_theory}

\begin{algorithm}[t]
\caption{TUP offline BoN distillation with BCE}
\label{alg:bce_alignment}
\begin{algorithmic}[1] % [1] enables automatic line numbering
    \Require Offline prompts $\{x_i\sim\mathcal{D}\}_{i=1}^n$, reference policy $\pi_{\mathrm{ref}}$, reward model $r$, truncation $\lambda$, sharpness $\beta$
    \Require For each $x_i$, a pool $\{y_{i,j}\sim \pi_{\mathrm{ref}}(\cdot \mid x_i)\}_{j=1}^K$ with empirical in-pool win-rates
    \[
        \hat{w}_r(x_i,y_{i,j})
        =
        \frac{1 + \sum_{\ell \ne j}\mathds{1}[r(x_i,y_{i,j}) \ge r(x_i,y_{i,\ell})]}{K}
        \in \left\{\frac{1}{K},\dots,1\right\}
    \]
    
    \State Precompute truncated win-rates and bias term:
    \[
        \hat{w}_{\lambda,r}(x_i,y_{i,j}) = \max\!\bigl(\hat{w}_r(x_i,y_{i,j}) - \lambda,\, 0\bigr),
        \qquad
        b_{\lambda,\beta} = \beta \log Z_{\lambda,\beta}
    \]
    
    \For{each training step}
        \State Sample a minibatch of indexed pool elements $(x_i,y_{i,j})$ from the reference-sampled pool
        \State Compute the logit and predicted probability:
        \begin{equation*}
        s_\theta(x_i,y_{i,j}) = \beta \log \frac{\pi_\theta(y_{i,j}\mid x_i)}{\pi_{\mathrm{ref}}(y_{i,j}\mid x_i)} + b_{\lambda,\beta},
        \qquad
        p_\theta(x_i,y_{i,j}) = \sigma\!\bigl(s_\theta(x_i,y_{i,j})\bigr)
        \end{equation*}
        \State Take a gradient step to minimize the BCE loss:
        \begin{equation*}
        \mathcal{L}_{\mathrm{BCE}}(\theta)
        = - \hat{w}_{\lambda,r}(x_i,y_{i,j})\log p_\theta(x_i,y_{i,j})
        - \bigl(1-\hat{w}_{\lambda,r}(x_i,y_{i,j})\bigr)\log\!\bigl(1-p_\theta(x_i,y_{i,j})\bigr)
        \end{equation*}

    \EndFor
\end{algorithmic}
\end{algorithm}

We now provide a theoretical analysis that sheds light on the roles of truncation and reweighting in our target policy. To set the stage, fix a prompt \(x\), and let \(Y\sim\pi_{\mathrm{ref}}(\cdot\mid x)\). Recall that $w_r(x,Y)$ is the win-rate computed using the proxy reward \(r(x,\cdot)\). Define also the
\emph{oracle win-rate} 
\[
w_u(x,Y)
:=
\mathbb P_{Y'\sim\pi_{\mathrm{ref}}(\cdot\mid x)}
\!\bigl(u(x,Y')\le u(x,Y)\bigr),
\]
where \(u(x,\cdot)\) is an unknown oracle reward function.

The following analysis builds on the general formulation of rank-based policies $\pi_g$ from \eqref{eq:pi_g}. Notably, our proposed policy also falls under this construction, with the function \(g_{\lambda,\beta}(w)\propto \bigl(\frac{w-\lambda}{1-w+\lambda}\bigr)^{1/\beta}\mathds{1}\{w>\lambda\}\). With this in place, we can evaluate a candidate policy $\pi_g$ using the unknown oracle reward, through the expected oracle win-rate of completions $Y$ sampled from $\pi_g$:
\begin{equation*}
    \mathcal{U}_x(g) := \mathbb{E}_{Y \sim \pi_g(\cdot \mid x)} [w_u(x,Y)].
\end{equation*}

Our first result shows that, among monotone reweightings of proxy rank, it is enough to consider hard truncation rules. Let \(\mathcal F_\uparrow\) denote the set of nondecreasing densities on \([0,1]\), and for each \(\lambda\in[0,1)\), let
$
\nu_\lambda(w):=\frac{1}{1-\lambda}\mathds{1}\{w\in[\lambda,1]\}
$
denote the uniform density on the upper tail, which is the limit
\(g_{\lambda,\infty}(w)\propto \mathds{1}\{w>\lambda\}\). With this notation, we state the following
result.
\begin{theorem}[Informal]
\label{thm:informal_truncation}
Fix a prompt $x$, and assume that the proxy reward $r(x,Y)$ is continuous under $Y\sim\pi_{\mathrm{ref}}(\cdot\mid x)$, so that the proxy win-rate $w_r(x,Y)$ is uniform on $[0,1]$. Let $\mathcal F_\uparrow$ be the class of nonnegative nondecreasing densities on $[0,1]$. Then
\[
\sup_{f\in\mathcal F_\uparrow}\mathcal U_x(f)
=
\sup_{\lambda\in[0,1)}\mathcal U_x(\nu_\lambda).
\]
\end{theorem}

The formal statement and proof are given in
Appendix~\ref{app:truncation_monotone_reweighting}. This result implies that the
best hard-threshold rule, obtained by optimizing $\lambda$ for each prompt $x$, can match the
best rank-based rule in those families, such as QRPO and BoNBoN. In practice, tuning $\lambda$ for each prompt $x$ is infeasible, and thus we treat it as a global quality floor. To empirically quantify the potential gain of prompt-adaptive truncation over a fixed global threshold, we consider a simplified, training-free experiment that treats a strong reward model as an ``oracle'' evaluator. We show that the difference is relatively insignificant. For full details, refer to Section~\ref{appx:global_vs_per_prompt} of the Appendix.

Once the threshold is fixed, we show that tail sharpening can further improve the policy if the proxy win rates within the retained tail is informative about oracle preferences. To formalize this relationship, define the oracle--proxy profile as
$m_x(w):=\mathbb E_{Y\sim\pi_{\mathrm{ref}}(\cdot\mid x)}
[w_u(x,Y)\mid w_r(x,Y)=w]$,
the average oracle win rate among completions with proxy win rate $w$. In oracle rank space, the following result states that a finite $\beta$ can improve over pure truncation when this profile is positively aligned with the within-tail log-odds tilt. The formal statement and proof are given in Appendix~\ref{app:general_within_tail_upweighting}.

\begin{proposition}[Informal]
\label{prop:informal_beta}
Fix a threshold $\lambda_0\in(0,1)$. If, within the retained tail $[\lambda_0,1]$, the oracle profile $m_x$ has positive covariance with the within-tail log-odds tilt $\log((w-\lambda_0)/(1-w+\lambda_0))$, then there exists a finite $\beta_0 \in (0,\infty)$ such that $\pi^\star_{\lambda_0,\beta_0}$ achieves strictly larger expected oracle win-rate than pure truncation policy $\pi^\star_{\lambda_0,\infty}$.
\end{proposition}

As with the truncation threshold $\lambda$, the theoretically optimal tail-sharpening parameter $\beta$ is prompt-specific. In practice, however, we use a single global value of $\beta$ across all prompts. In this sense, Theorem~\ref{thm:informal_truncation} and Proposition~\ref{prop:informal_beta} provide idealized theoretical support for separating truncation from sharpness, but they do not directly explain how a training model can benefit from it.

To demonstrate the effects of threshold $\lambda$ and the reweighting effect of a finite $\beta$, Appendix~\ref{app:quadratic_worked_example} and Figure~\ref{fig:theory} use a stylized, unnormalized quadratic oracle--proxy example; the appendix shows that normalization does not affect the relevant comparisons. The improvement of the pure-truncation policy $\pi^\star_{\lambda,\infty}$ over QRPO reflects the gain from using a threshold to remove low-ranked completions. The improvement of the best finite-$\beta$ TUP policy over $\pi^\star_{\lambda,\infty}$ reflects the additional gain from reweighting within the retained tail once the threshold is fixed.

\subsection{Training objective: policy learning as a classification problem}
\label{sec:training}

Having established the theoretical properties of our proposed policy $\pi_{\lambda,\beta}^\star$, we now turn to formulate the training objective induced by this target. For completions in the support of $\pi_{\lambda,\beta}^\star$, rearranging~\eqref{eq:optimal_policy_lambda} gives
\begin{equation}
    \label{eq:logit_relation_lambda}
    w_{\lambda,r}(x,y)
    \;=\;
     \sigma\left( \beta\log \frac{\pi_{\lambda,\beta}^\star(y \mid x)}{\pi_{\mathrm{ref}}(y \mid x)}
    +
    \beta \log Z_{\lambda,\beta}\right),
\end{equation}
where $\sigma$ is the sigmoid function and $\mathrm{\logit}$ is its inverse on $(0,1)$, so $\sigma(\mathrm{\logit}(x)) = x$ for $x\in(0,1)$, with endpoint limits $0$ and $1$ as $x\to0^+$ and $x\to1^-$.
The above relation suggests that we can fit a parameterized policy $\pi_\theta$ that estimates $\pi_{\lambda,\beta}^\star$ using logistic regression of the shifted-truncated win-rate. Indeed, we can define our classifier as
\begin{equation}
    p_\theta(x,y)
    \;:=\;
    \sigma\left(\beta \log \frac{\pi_\theta(y \mid x)}{\pi_{\mathrm{ref}}(y \mid x)}
    +
   \beta \log Z_{\lambda,\beta}\right),
\end{equation}
where $\beta \log Z_{\lambda,\beta}$ acts as a constant and known intercept. For empirical finite-pool ranks, the normalizer has an exact finite-$K$ analogue of the same form, with the integral replaced by a discrete average over rank locations. As in QRPO, our reported experiments use the population intercept, while Appendix~\ref{app:normalizer_incomplete_beta} gives the finite-pool alternative. Empirically, the population and finite-pool normalizers performs similarly to the analytical normalizer, except at extremely small $\beta$, where the finite-pool form is needed to keep the normalizer finite. 

The construction above naturally results in a fully offline binary cross-entropy (BCE) objective with the shifted-truncated win-rate as a soft label:
\begin{tcolorbox}[colback=boxbg, colframe=black, boxrule=1pt, arc=10pt]
The \textbf{BCE TUP distillation} objective is:
\begin{equation}
    \label{eq:bce_objective}
    \mathcal{L}(\theta)
    \;=\;
    \mathbb{E}_{x\sim \mathcal D, y\sim\pi_{\mathrm{ref}}(\cdot \mid x)}
    \Bigl[
    -w_{\lambda,r}(x,y)\log \!\bigl(p_\theta(x,y)\bigr)
    -(1-w_{\lambda,r}(x,y))\log\bigl(1-p_\theta(x,y)\bigr)
    \Bigr]. 
\end{equation}
\end{tcolorbox}

This loss requires only precomputed scalar rank labels $w_{\lambda,r}(x,y)$ and the global intercept $\beta \log Z_{\lambda,\beta}$. It requires no pairwise preferences, no online sampling, and no runtime estimation of a prompt-dependent partition function.

Algorithm~\ref{alg:bce_alignment} summarizes the resulting fully offline training procedure. In practice, for each prompt we first sample a reference pool, compute each completion's empirical in-pool win-rate $\hat{w}_r(x,y)$ against the other $K-1$ completions in that same pool, apply the global truncation $\lambda$, and then fit $\pi_\theta$ by binary cross-entropy on the distilled-to-reference log-likelihood ratio.

\section{Experiments}
\label{sec:exps}

\begin{table}[!b]
\centering
\caption{In-dataset evaluation for Llama 8B Tülu 3 SFT trained separately on UltraFeedback and Magpie Air. We report the LC reward metric for each reward model.}
\label{tab:llama_indataset_lc_reward}
\resizebox{\textwidth}{!}{%
\begin{tabular}{l ccc @{\hspace{1.0em}} ccc}
\toprule
& \multicolumn{3}{c}{\textbf{UltraFeedback eval split}} 
& \multicolumn{3}{c}{\textbf{Magpie Air eval split}} \\
\cmidrule(lr){2-4} \cmidrule(lr){5-7}
Method 
& ArmoRM & Skywork-Llama & Skywork-Qwen
& ArmoRM & Skywork-Llama & Skywork-Qwen \\
\midrule

Initial & $0.1300{\scriptstyle\pm 0.0008}$ & $9.18{\scriptstyle\pm 0.24}$ & $3.80{\scriptstyle\pm 0.19}$ & $0.1531{\scriptstyle\pm 0.0012}$ & $10.14{\scriptstyle\pm 0.30}$ & $5.08{\scriptstyle\pm 0.25}$ \\

\midrule

DPO & $0.1499{\scriptstyle\pm 0.0001}$ & $16.16{\scriptstyle\pm 0.06}$ & $7.97{\scriptstyle\pm 0.03}$ & $0.1727{\scriptstyle\pm 0.0005}$ & $16.33{\scriptstyle\pm 0.15}$ & $9.89{\scriptstyle\pm 0.05}$ \\

REBEL & $0.1494{\scriptstyle\pm 0.0003}$ & $16.37{\scriptstyle\pm 0.14}$ & $8.01{\scriptstyle\pm 0.07}$ & $0.1709{\scriptstyle\pm 0.0007}$ & $17.32{\scriptstyle\pm 0.34}$ & $9.94{\scriptstyle\pm 0.15}$ \\

REBEL (random) & $0.1496{\scriptstyle\pm 0.0005}$ & $16.84{\scriptstyle\pm 0.14}$ & $8.14{\scriptstyle\pm 0.08}$ & $0.1718{\scriptstyle\pm 0.0001}$ & $15.08{\scriptstyle\pm 0.19}$ & $9.27{\scriptstyle\pm 0.06}$ \\

QRPO & $0.1521{\scriptstyle\pm 0.0003}$ & $17.20{\scriptstyle\pm 0.29}$ & $8.22{\scriptstyle\pm 0.08}$ & $0.1748{\scriptstyle\pm 0.0003}$ & $16.32{\scriptstyle\pm 0.29}$ & $9.07{\scriptstyle\pm 0.18}$ \\

QRPO (random) & $\mathbf{0.1522{\scriptstyle\pm 0.0004}}$ & $16.92{\scriptstyle\pm 0.08}$ & $8.05{\scriptstyle\pm 0.07}$ & $0.1695{\scriptstyle\pm 0.0005}$ & $15.31{\scriptstyle\pm 0.30}$ & $8.54{\scriptstyle\pm 0.24}$ \\

BoNBoN & $0.1409{\scriptstyle\pm 0.0003}$ & $12.74{\scriptstyle\pm 0.05}$ & $5.92{\scriptstyle\pm 0.04}$ & $0.1641{\scriptstyle\pm 0.0008}$ & $13.76{\scriptstyle\pm 0.20}$ & $7.60{\scriptstyle\pm 0.07}$ \\

\midrule

\textbf{TUP mild} & $0.1500{\scriptstyle\pm 0.0006}$ & $17.47{\scriptstyle\pm 0.28}$ & $8.38{\scriptstyle\pm 0.12}$ & $0.1798{\scriptstyle\pm 0.0006}$ & $20.68{\scriptstyle\pm 0.46}$ & $12.27{\scriptstyle\pm 0.25}$ \\

\textbf{TUP mid.} & $0.1499{\scriptstyle\pm 0.0002}$ & $17.45{\scriptstyle\pm 0.13}$ & $8.39{\scriptstyle\pm 0.07}$ & $\mathbf{0.1806{\scriptstyle\pm 0.0013}}$ & $\mathbf{21.24{\scriptstyle\pm 0.20}}$ & $\mathbf{12.49{\scriptstyle\pm 0.11}}$ \\

\textbf{TUP aggressive} & $0.1499{\scriptstyle\pm 0.0003}$ & $\mathbf{17.54{\scriptstyle\pm 0.16}}$ & $\mathbf{8.43{\scriptstyle\pm 0.05}}$ & $0.1794{\scriptstyle\pm 0.0037}$ & $19.64{\scriptstyle\pm 0.58}$ & $11.38{\scriptstyle\pm 0.24}$ \\

\bottomrule
\end{tabular}}
\end{table}
\begin{table}[t]
\centering
\caption{AlpacaEval performance for Llama 8B Tülu 3 SFT trained on UltraFeedback (top) and Magpie Air (bottom). We report the LC reward metric for each reward model, and win-rate using \texttt{gpt-4o} judge results; length is measured in characters. Methods marked ``random'' are trained on a random pair of completions from each pool, whereas unmarked pairwise methods use the best and worst completions. TUP is trained on a single random completion per pool.}
\label{tab:llama_alpacaeval}
\vspace{5pt}
\small
\setlength{\tabcolsep}{4.5pt}
\renewcommand{\arraystretch}{1.05}
\resizebox{0.9\textwidth}{!}{%
\begin{tabular}{c l ccc ccc}
\toprule
& & \multicolumn{3}{c}{Reward model} & \multicolumn{3}{c}{GPT Judge} \\
\cmidrule(lr){3-5} \cmidrule(lr){6-8}
Train set & Method 
& ArmoRM & Skywork-Llama & Skywork-Qwen 
& LC Win & Win & Len \\
\midrule

\multirow{10}{*}{\rotatebox[origin=c]{90}{UltraFeedback}} & Initial & $0.1448{\scriptstyle\pm 0.0030}$ & $13.44{\scriptstyle\pm 1.60}$ & $6.65{\scriptstyle\pm 0.89}$ & $15.18{\scriptstyle\pm 0.36}$ & $8.57{\scriptstyle\pm 0.98}$ & 1011 \\
\cmidrule[0.01pt](lr){2-8}
 & DPO & $0.1619{\scriptstyle\pm 0.0001}$ & $21.59{\scriptstyle\pm 0.16}$ & $10.91{\scriptstyle\pm 0.08}$ & $\mathbf{42.30{\scriptstyle\pm 0.16}}$ & $37.08{\scriptstyle\pm 1.70}$ & 1807 \\
 & REBEL & $0.1620{\scriptstyle\pm 0.0002}$ & $22.15{\scriptstyle\pm 0.12}$ & $11.19{\scriptstyle\pm 0.08}$ & $41.81{\scriptstyle\pm 0.12}$ & $39.81{\scriptstyle\pm 1.73}$ & 1938 \\
 & REBEL (random) & $0.1625{\scriptstyle\pm 0.0003}$ & $22.50{\scriptstyle\pm 0.18}$ & $11.31{\scriptstyle\pm 0.05}$ & $\mathbf{42.44{\scriptstyle\pm 0.15}}$ & $40.93{\scriptstyle\pm 1.73}$ & 1959 \\
 & QRPO & $\mathbf{0.1668{\scriptstyle\pm 0.0000}}$ & $20.69{\scriptstyle\pm 0.19}$ & $10.23{\scriptstyle\pm 0.13}$ & $39.64{\scriptstyle\pm 0.16}$ & $\mathbf{53.17{\scriptstyle\pm 1.76}}$ & 3100 \\
 & QRPO (random) & $0.1663{\scriptstyle\pm 0.0004}$ & $22.26{\scriptstyle\pm 0.02}$ & $11.11{\scriptstyle\pm 0.03}$ & $\mathbf{42.35{\scriptstyle\pm 0.13}}$ & $48.51{\scriptstyle\pm 1.76}$ & 2436 \\
 & BoNBoN & $0.1503{\scriptstyle\pm 0.0028}$ & $16.75{\scriptstyle\pm 0.30}$ & $8.44{\scriptstyle\pm 0.19}$ & $23.28{\scriptstyle\pm 0.43}$ & $15.40{\scriptstyle\pm 1.27}$ & 1325 \\
\cmidrule[0.5pt](lr){2-8}
 & \textbf{TUP mild} & $0.1638{\scriptstyle\pm 0.0000}$ & $22.85{\scriptstyle\pm 0.37}$ & $11.41{\scriptstyle\pm 0.15}$ & $39.41{\scriptstyle\pm 0.14}$ & $49.50{\scriptstyle\pm 1.76}$ & 2730 \\
 & \textbf{TUP mid.} & $0.1633{\scriptstyle\pm 0.0005}$ & $\mathbf{23.05{\scriptstyle\pm 0.23}}$ & $\mathbf{11.54{\scriptstyle\pm 0.05}}$ & $40.27{\scriptstyle\pm 0.06}$ & $49.13{\scriptstyle\pm 1.76}$ & 2670 \\
 & \textbf{TUP aggressive} & $0.1638{\scriptstyle\pm 0.0004}$ & $22.90{\scriptstyle\pm 0.31}$ & $11.46{\scriptstyle\pm 0.19}$ & $\mathbf{42.36{\scriptstyle\pm 0.14}}$ & $51.24{\scriptstyle\pm 1.76}$ & 2709 \\

\midrule

\midrule

\multirow{10}{*}{\rotatebox[origin=c]{90}{Magpie Air}} & Initial & $0.1448{\scriptstyle\pm 0.0030}$ & $13.44{\scriptstyle\pm 1.60}$ & $6.65{\scriptstyle\pm 0.89}$ & $15.18{\scriptstyle\pm 0.36}$ & $8.57{\scriptstyle\pm 0.98}$ & 1011 \\
\cmidrule[0.2pt](lr){2-8}
 & DPO & $0.1581{\scriptstyle\pm 0.0008}$ & $18.94{\scriptstyle\pm 0.16}$ & $9.73{\scriptstyle\pm 0.06}$ & $32.11{\scriptstyle\pm 0.25}$ & $26.15{\scriptstyle\pm 1.55}$ & 1644 \\
 & REBEL & $0.1591{\scriptstyle\pm 0.0007}$ & $20.65{\scriptstyle\pm 0.19}$ & $10.37{\scriptstyle\pm 0.12}$ & $39.66{\scriptstyle\pm 0.21}$ & $36.58{\scriptstyle\pm 1.70}$ & 1872 \\
 & REBEL (random) & $0.1565{\scriptstyle\pm 0.0002}$ & $17.81{\scriptstyle\pm 0.14}$ & $9.15{\scriptstyle\pm 0.07}$ & $31.37{\scriptstyle\pm 0.26}$ & $26.21{\scriptstyle\pm 1.55}$ & 1675 \\
 & QRPO & $0.1578{\scriptstyle\pm 0.0000}$ & $19.00{\scriptstyle\pm 0.08}$ & $9.30{\scriptstyle\pm 0.00}$ & $32.39{\scriptstyle\pm 0.13}$ & $32.48{\scriptstyle\pm 1.65}$ & 1992 \\
 & QRPO (random) & $0.1553{\scriptstyle\pm 0.0009}$ & $18.64{\scriptstyle\pm 0.31}$ & $9.24{\scriptstyle\pm 0.15}$ & $27.97{\scriptstyle\pm 0.27}$ & $23.54{\scriptstyle\pm 1.49}$ & 1628 \\
 & BoNBoN & $0.1446{\scriptstyle\pm 0.0011}$ & $14.92{\scriptstyle\pm 0.22}$ & $7.32{\scriptstyle\pm 0.17}$ & $19.82{\scriptstyle\pm 0.42}$ & $13.73{\scriptstyle\pm 1.21}$ & 1346 \\
\cmidrule[0.5pt](lr){2-8}
 & \textbf{TUP mild} & $0.1585{\scriptstyle\pm 0.0003}$ & $20.58{\scriptstyle\pm 0.36}$ & $10.21{\scriptstyle\pm 0.15}$ & $36.82{\scriptstyle\pm 0.02}$ & $\mathbf{43.29{\scriptstyle\pm 1.75}}$ & 2712 \\
 & \textbf{TUP mid.} & $0.1585{\scriptstyle\pm 0.0006}$ & $\mathbf{21.08{\scriptstyle\pm 0.44}}$ & $\mathbf{10.44{\scriptstyle\pm 0.08}}$ & $35.81{\scriptstyle\pm 0.04}$ & $42.55{\scriptstyle\pm 1.74}$ & 2645 \\
 & \textbf{TUP aggressive} & $0.1574{\scriptstyle\pm 0.0006}$ & $19.76{\scriptstyle\pm 0.39}$ & $9.77{\scriptstyle\pm 0.13}$ & $32.31{\scriptstyle\pm 0.03}$ & $38.94{\scriptstyle\pm 1.72}$ & 2626 \\
\bottomrule
\end{tabular}}
\end{table}

\paragraph{Experimental setup.}
We build on the QRPO experimental framework of~\citet{matrenok_quantile_2025}, which evaluates offline alignment methods across multiple models and datasets. We compare TUP with six baselines across two model families, two datasets, three reward-model evaluators, and a GPT judge. We use UltraFeedback~\citep{cui2024ultrafeedback} and Magpie Air~\citep{xu2025magpie} as training datasets. In our experiments, we apply each alignment method to Llama 8B Tülu 3 SFT~\citep{lambert2025tulu} separately on both datasets and to Mistral-7B~\citep{jiang2023mistral7b} on Magpie Air; for Mistral, we first perform dedicated supervised fine-tuning (SFT) to obtain the initial reference policy.

Following QRPO, we use their published datasets, which use ArmoRM~\citep{wang_interpretable_2024} as the offline reward model for scoring the reference-model training completions. To test generalization beyond the training reward model, we evaluate using two Skywork-v2 reward models, Skywork-Reward-V2-Llama-3.1-8B and Skywork-Reward-V2-Qwen3-8B~\citep{liu_skywork-reward-v2_2025}, referred to as Skywork-Llama and Skywork-Qwen respectively. These models are ranked \#1 and \#2, respectively, among available classifier reward models on the RewardBench leaderboard at the time of writing~\citep{lambert_rewardbench_2025}.\footnote{\url{https://huggingface.co/spaces/allenai/reward-bench}} We additionally evaluate on AlpacaEval using \texttt{gpt-4o} as a judge~\citep{zheng2023judging, hurst2024gpt}. Table~\ref{tab:rewardbench} of the Appendix reports the RewardBench standings of all reward and judge models in our experiments. We report length-controlled (LC) rewards in the main tables; raw reward results are given in Appendix~\ref{appx:raw_rewards}.

The most natural baselines for our setting are rank-based distillation methods. We include BoNBoN~\citep{gui_bonbon_2024} and the QRPO variant that computes the loss with random completions~\citep{matrenok_quantile_2025}, denoted QRPO (random). While less related, we also compare to strong offline-alignment baselines using pairwise or relative-reward signals, namely DPO~\citep{rafailov_direct_2023} and REBEL~\citep{gao_rebel_2024}. Unless marked random, DPO and REBEL use the best--worst pair from each pool; REBEL (random) uses random pairs. We also include QRPO trained on best--worst pairs. For all DPO, REBEL, and QRPO variants, we adopt the best-performing configurations from the \emph{off-policy} setting reported in~\citet{matrenok_quantile_2025}. To avoid clutter, tables omit the best--worst label when it is the default. We do not include SimPO~\citep{meng_simpo_2024}, whose length-normalized objective makes it structurally different from baselines that do not explicitly account for length. Because BoNBoN is not covered by the QRPO benchmark, we run a dedicated hyperparameter search and select the best checkpoint by validation reward. Full implementation details are provided in Appendix~\ref{appx:imp}.

In all tables, unless stated otherwise, we show TUP with 3 different global truncation values; $\lambda=0.2$ (mild) which truncates only the worst rollout out of the 6 completions,  $\lambda=0.5$ (mid.) and $\lambda=0.8$ (aggressive) that keeps only the top two completions. 

\begin{table}[th]
\centering
\caption{LC rewards for Mistral 7B SFT, trained on Magpie Air. Tested on in-dataset test split (MA) and AlpacaEval (AE) evaluated by ArmoRM, Skywork-Llama and Skywork-Qwen reward models, with AlpacaEval also judged by \texttt{gpt-4o}.}
\label{tab:mistral_sft_magpie_lc_reward}
\resizebox{\textwidth}{!}{%
\begin{tabular}{l cc cc cc c c c}
\toprule
& \multicolumn{2}{c}{ArmoRM} & \multicolumn{2}{c}{Skywork-v2-Llama} & \multicolumn{2}{c}{Skywork-v2-Qwen} & \multicolumn{3}{c}{GPT Judge} \\
\cmidrule(lr){2-3} \cmidrule(lr){4-5} \cmidrule(lr){6-7} \cmidrule(lr){8-10}
Method & MA & AE & MA & AE & MA & AE & LC Win & Win & Len \\
\midrule
Initial & $0.1646{\scriptstyle\pm 0.0006}$ & $0.1477{\scriptstyle\pm 0.0002}$ & $17.36{\scriptstyle\pm 0.11}$ & $16.53{\scriptstyle\pm 0.19}$ & $11.57{\scriptstyle\pm 0.07}$ & $8.76{\scriptstyle\pm 0.09}$ & $21.94{\scriptstyle\pm 0.27}$ & $18.76{\scriptstyle\pm 1.38}$ & $1736$ \\
\midrule
DPO  & $0.1792{\scriptstyle\pm 0.0008}$ & $0.1583{\scriptstyle\pm 0.0010}$ & $22.62{\scriptstyle\pm 0.22}$ & $21.50{\scriptstyle\pm 0.11}$ & $16.94{\scriptstyle\pm 0.19}$ & $12.43{\scriptstyle\pm 0.05}$ & $31.80{\scriptstyle\pm 0.05}$ & $37.14{\scriptstyle\pm 1.70}$ & $2860$ \\
REBEL & $0.1732{\scriptstyle\pm 0.0021}$ & $0.1462{\scriptstyle\pm 0.0016}$ & $20.83{\scriptstyle\pm 0.16}$ & $13.58{\scriptstyle\pm 0.19}$ & $14.08{\scriptstyle\pm 0.15}$ & $7.51{\scriptstyle\pm 0.15}$ & $22.11{\scriptstyle\pm 0.02}$ & $19.25{\scriptstyle\pm 1.39}$ & $8470$ \\
REBEL (random) & $0.1822{\scriptstyle\pm 0.0026}$ & $0.1588{\scriptstyle\pm 0.0004}$ & $\mathbf{24.64{\scriptstyle\pm 0.23}}$ & $\mathbf{21.91{\scriptstyle\pm 0.46}}$ & $\mathbf{18.32{\scriptstyle\pm 0.26}}$ & $\mathbf{12.67{\scriptstyle\pm 0.18}}$ & $34.45{\scriptstyle\pm 0.12}$ & $40.12{\scriptstyle\pm 1.73}$ & $3013$ \\
QRPO & $0.1720{\scriptstyle\pm 0.0005}$ & $0.1565{\scriptstyle\pm 0.0002}$ & $21.09{\scriptstyle\pm 0.12}$ & $20.29{\scriptstyle\pm 0.41}$ & $14.23{\scriptstyle\pm 0.03}$ & $11.02{\scriptstyle\pm 0.14}$ & $31.97{\scriptstyle\pm 0.19}$ & $34.78{\scriptstyle\pm 1.68}$ & $2261$ \\
QRPO (random) & $0.1858{\scriptstyle\pm 0.0013}$ & $0.1425{\scriptstyle\pm 0.0078}$ & $21.96{\scriptstyle\pm 0.22}$ & $17.23{\scriptstyle\pm 1.53}$ & $15.86{\scriptstyle\pm 0.20}$ & $9.54{\scriptstyle\pm 1.04}$ & $27.53{\scriptstyle\pm 0.12}$ & $29.07{\scriptstyle\pm 1.60}$ & $4740$ \\
BoNBoN & $0.1667{\scriptstyle\pm 0.0004}$ & $0.1494{\scriptstyle\pm 0.0006}$ & $18.40{\scriptstyle\pm 0.32}$ & $17.46{\scriptstyle\pm 0.13}$ & $12.35{\scriptstyle\pm 0.35}$ & $9.34{\scriptstyle\pm 0.09}$ & $24.85{\scriptstyle\pm 0.21}$ & $22.11{\scriptstyle\pm 1.46}$ & $1828$ \\
\midrule
\textbf{TUP mild} & $\mathbf{0.1864{\scriptstyle\pm 0.0004}}$ & $0.1612{\scriptstyle\pm 0.0032}$ & $23.47{\scriptstyle\pm 0.15}$ & $21.04{\scriptstyle\pm 0.56}$ & $17.01{\scriptstyle\pm 0.10}$ & $11.75{\scriptstyle\pm 0.35}$ & $\mathbf{36.84{\scriptstyle\pm 0.04}}$ & $\mathbf{40.62{\scriptstyle\pm 1.73}}$ & $3382$ \\
\textbf{TUP mid.} & $0.1832{\scriptstyle\pm 0.0003}$ & $0.1620{\scriptstyle\pm 0.0004}$ & $22.81{\scriptstyle\pm 0.18}$ & $21.34{\scriptstyle\pm 0.34}$ & $16.96{\scriptstyle\pm 0.08}$ & $12.25{\scriptstyle\pm 0.11}$ & $35.69{\scriptstyle\pm 0.05}$ & $39.25{\scriptstyle\pm 1.72}$ & $2765$ \\
\textbf{TUP aggressive} & $0.1826{\scriptstyle\pm 0.0002}$ & $\mathbf{0.1625{\scriptstyle\pm 0.0024}}$ & $21.32{\scriptstyle\pm 0.12}$ & $20.50{\scriptstyle\pm 0.76}$ & $16.00{\scriptstyle\pm 0.02}$ & $11.54{\scriptstyle\pm 0.34}$ & $33.22{\scriptstyle\pm 0.14}$ & $39.63{\scriptstyle\pm 1.73}$ & $2688$ \\
\bottomrule
\end{tabular}}
\end{table}

\paragraph{Results.}
Table~\ref{tab:llama_indataset_lc_reward} reports two in-dataset experiments on Llama 8B, where each method is evaluated on the held-out split of its respective training dataset, UltraFeedback~\citep{cui2024ultrafeedback} or Magpie Air~\citep{xu2025magpie}. As shown, TUP is competitive with strong offline alignment baselines on both datasets. Under ArmoRM, all methods achieve broadly similar results. Larger differences appear under the two Skywork reward models. On both UltraFeedback and Magpie Air, TUP obtains the best in-dataset results under both Skywork-Llama and Skywork-Qwen. If these reward models are viewed as different proxies for the oracle reward, the results suggest that TUP generalizes well across proxy rewards compared to the baselines considered.

Table~\ref{tab:llama_alpacaeval} reports AlpacaEval results for the above mentioned trained Llama models. Under both datasets, TUP obtains the best results under both Skywork-Llama and Skywork-Qwen, including on AlpacaEval LC reward, and remains highly competitive under the GPT judge. While the Magpie-Air-trained models on AlpacaEval results are more mixed, TUP remains competitive. The supplementary raw reward tables in Appendix~\ref{appx:raw_rewards} provide the corresponding non-LC scores and in-dataset response lengths, showing that the main trends are not limited to the LC metric. Table~\ref{tab:mistral_sft_magpie_lc_reward}, which reports both in-dataset and AlpacaEval LC rewards, shows that TUP remains competitive when applied to the Mistral model family.

In addition to the LC-reward results reported above, we conduct a pairwise length-matched reward comparison to further isolate TUP's reward gains from the effect of its longer average responses. As reported in Table~\ref{tab:length-matched} of the Appendix, TUP responses are largely preferred against similar-length responses from the baseline methods on the same prompt. This provides additional evidence that our advantage does not stem from response length alone. Full details are in Section~\ref{appx:length-matched} of the Appendix.

Figure~\ref{fig:lambda_beta_sweep} shows the effects of TUP's truncation threshold $\lambda$ and preference sharpness $\beta$, using the same validation-based checkpoint-selection protocol as the main experiments (Tables~\ref{tab:llama_indataset_lc_reward}--\ref{tab:llama_magpie_raw_reward}). To approximate the no-truncation setting, we use $\lambda=0.2$, which removes only the lowest-ranked of the six responses, a response that would otherwise receive an extremely low weight. To approximate ``pure'' truncation, we use $\beta=100$, which produces an approximately uniform weighting over the retained tail while still remaining finite and allowing for some deviation from $\piref$. Generally, intermediate values of both $\lambda$ and $\beta$ tend to yield the highest LC rewards. These results suggest that the best performance comes from combining moderate truncation with non-uniform within-tail reweighting, supporting our choice to treat $\lambda$ and $\beta$ as separate parameters.

\begin{figure}[tp]
     \centering
     % First Subfigure
     \begin{subfigure}[t]{0.80\textwidth}
         \centering
         \includegraphics[width=\textwidth]{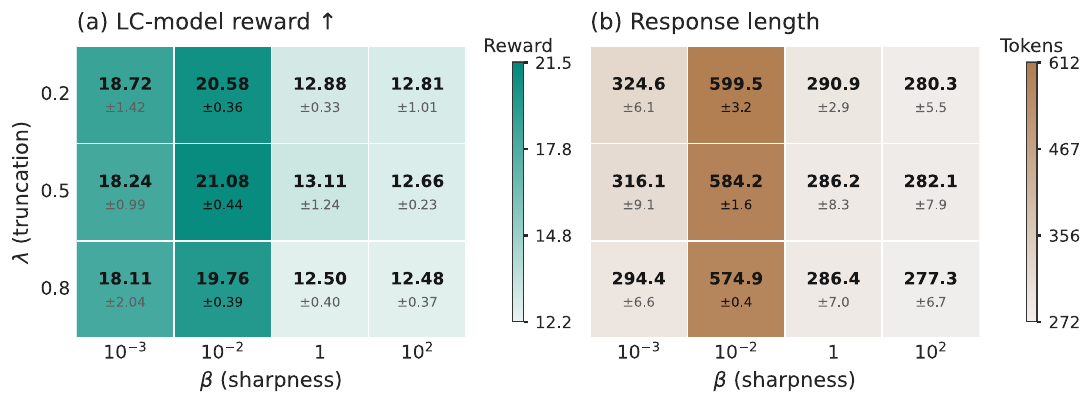}
     \end{subfigure}
    \caption{Ablation study of TUP's truncation threshold $\lambda$ and sharpness $\beta$ for Llama 8B Tülu 3 SFT trained on Magpie Air and evaluated on AlpacaEval with Skywork-Llama reward model. TUP uses $K=6$ reference completions per prompt. Cells report the mean and sample standard deviation across three generation seeds. \textbf{Left:} LC rewards; higher is better. \textbf{Right:} Mean response length in tokens.}
\label{fig:lambda_beta_sweep}
\end{figure}

\section{Discussion}
\label{sec:discussion}

One limitation of our method is that both $\lambda$ and $\beta$ must be tuned based on validation performance, whereas other baselines tune only $\beta$. This may increase the computational cost when a suitable initial range for $\lambda$ is unknown. A practical way to narrow this range is to construct a plot such as the right panel of Figure~\ref{fig:top_bot_disagreement}. Given a dataset, a proxy reward model, and an auxiliary reward model treated as an ``oracle,'' this plot provides an efficient way to identify a promising range of truncation levels. Once a truncation level that retains responses ranked highly by the ``oracle'' while removing poorly ranked ones is identified, the remaining hyperparameter search is once again only on $\beta$ values.

TUP also leaves several directions for future work. First, although TUP improves rank-based policy distillation, it is still exposed to reward hacking; incorporating reward-hacking mitigation methods could further improve its performance. Second, future work could replace the fixed global $\lambda$ and $\beta$ with potentially learned prompt-specific parameters.

More broadly, TUP may support safety alignment when unsafe or otherwise undesirable completions are consistently assigned low proxy ranks, allowing truncation to remove them from the target support rather than merely down-weight them.

\bibliographystyle{plainnat}
\bibliography{references_}
\newpage

\appendix

\section{Theory and proofs}
\label{appx:proofs}

\subsection{Normalization constant for Proposition~\ref{prop:optimal_policy_lambda}}
\label{app:normalizer_incomplete_beta}

\begin{proof}
By the Gibbs solution in~\eqref{eq:gibbs_policy}, substituting $r=R_\lambda$ yields
\[
  \pi^\star(y \mid x)
  \;=\;
  \frac{1}{Z(x)}\,\pi_{\mathrm{ref}}(y \mid x)\exp\!\left(\frac{R_\lambda(x,y)}{\beta}\right)
  \;=\;
  \frac{1}{Z(x)}\,\pi_{\mathrm{ref}}(y \mid x)
  \left(\frac{w_{\lambda,r}(x,y)}{1-w_{\lambda,r}(x,y)}\right)^{1/\beta}.
\]
It therefore remains to identify the normalizer $Z(x)$ and show that it is independent of $x$. Under the induced uniform win-rate variable $w \sim \mathrm{Uniform}[0,1]$, the normalization term becomes
\begin{equation}
  Z(x)
  \;=\;
  \int_0^1
    \left(
      \frac{\max(w - \lambda,\, 0)}{1 - \max(w - \lambda,\, 0)}
    \right)^{\!1/\beta}
  dw.
\end{equation}
For $w \le \lambda$, the shifted win-rate is zero, so the integrand vanishes. Therefore,
\begin{equation}
  Z(x)
  \;=\;
  \int_\lambda^1
    \left(
      \frac{w - \lambda}{1 - (w - \lambda)}
    \right)^{\!1/\beta}
  dw.
\end{equation}
Applying the change of variables $u := w - \lambda$ gives
\begin{equation}
  Z(x)
  \;=\;
  \int_0^{1-\lambda}
    \left(\frac{u}{1-u}\right)^{\!1/\beta}
  du
  \;=\;
  \int_0^{1-\lambda}
    u^{\,1/\beta}(1-u)^{-1/\beta}
  \,du.
\end{equation}

Recalling that the incomplete Beta function is defined in Proposition~\ref{prop:optimal_policy_lambda} as $B_z(a,b) = \int_0^z t^{a-1}(1-t)^{b-1}\,dt$, we identify
\[
  a = 1 + \tfrac{1}{\beta},
  \qquad
  b = 1 - \tfrac{1}{\beta}.
\]
Hence
\begin{equation}
  Z(x)
  \;=\;
  \int_0^{1-\lambda}
    u^{\,(1+1/\beta)-1}(1-u)^{\,(1-1/\beta)-1}
  \,du
  \;=\;
    \mathrm{Beta}_{1-\lambda}\!\Bigl(1+\tfrac{1}{\beta},\, 1-\tfrac{1}{\beta}\Bigr)
  \;=\;
  Z_{\lambda,\beta},
\end{equation}
which is independent of $x$.
\end{proof}

\paragraph{Finite-pool normalizer.}
The closed-form normalizer in Proposition~\ref{prop:optimal_policy_lambda} is the population normalizer induced by the continuous win-rate $w_r(x,Y)$. In the experiments, labels are computed from finite prompt-specific pools. Under the empirical-rank convention in Algorithm~\ref{alg:bce_alignment}, the corresponding discrete prompt-independent normalizer is
\begin{equation*}
Z_{K,\lambda,\beta}
=
\frac{1}{K}\sum_{j=1}^{K}
\left(
\frac{(\frac{j}{K}-\lambda)_+}
{1-(\frac{j}{K}-\lambda)_+}
\right)^{1/\beta}.
\end{equation*}
This finite-pool form can replace $Z_{\lambda,\beta}$ in the BCE intercept without changing the labels, the truncation support, or the loss. In our experiments, we use the population normalizer $Z_{\lambda,\beta}$, following the continuous win-rate normalization used in QRPO; in our experiments, replacing it with the finite-pool expression led to negligible performance differences. In extremely small-$\beta$ regimes, where the continuous calculation can become numerically unstable, the discrete expression above provides the exact finite-$K$ alternative for this empirical-rank convention.

\subsection{Formal oracle rank-space results for Section~\ref{sec:oracle_win-rate_theory}}
\label{app:oracle_win-rate_theory}

This appendix formalizes the oracle rank-space discussion from Section~\ref{sec:oracle_win-rate_theory}. The aim is to separate two decisions that smooth full-support reweighting conflates: which proxy ranks should remain in the support at all, and how the retained ranks should be reweighted. The first result shows that, among monotone proxy-rank reweightings, the oracle-best rule is always a hard lower-tail truncation. The second result is a fixed-threshold refinement statement: once a truncation level has been chosen, finite within-tail reweighting can improve over pure truncation when oracle value remains positively aligned with proxy rank inside the retained tail. We close with the quadratic worked example behind Figure~\ref{fig:theory}.

Fix a prompt $x$, and let $Y \sim \pi_{\mathrm{ref}}(\cdot \mid x)$. Define the proxy and oracle population win-rates by
\begin{equation*}
w_{r}(x,y)
:=
\mathbb P_{Y' \sim \pi_{\mathrm{ref}}(\cdot\mid x)}
\!\bigl(r(x,Y') \le r(x,y)\bigr),
\qquad
w_u(x,y)
:=
\mathbb P_{Y' \sim \pi_{\mathrm{ref}}(\cdot\mid x)}
\!\bigl(u(x,Y') \le u(x,y)\bigr).
\end{equation*}
and write
\begin{equation*}
W_r := w_{r}(x,Y),
\qquad
W_u := w_u(x,Y).
\end{equation*}
Since BoN-style methods act on relative rank rather than raw reward, the natural policy class in this analysis reweights a completion according to its proxy win-rate. For any measurable $g:[0,1]\to[0,\infty)$ satisfying
\begin{equation*}
0 < \mathbb E_{\pi_{\mathrm{ref}}}[g(W_r)] < \infty,
\end{equation*}
define
\begin{equation*}
\pi_g(y\mid x)
:=
\frac{g(w_{r}(x,y))\,\pi_{\mathrm{ref}}(y\mid x)}
{\mathbb E_{Y' \sim \pi_{\mathrm{ref}}(\cdot\mid x)}[g(w_{r}(x,Y'))]},
\qquad
\mathcal U_x(g)
:=
\mathbb E_{Y \sim \pi_g(\cdot\mid x)}[W_u].
\end{equation*}
Finally, let
\begin{equation*}
m_x(w) := \mathbb E[W_u \mid W_r = w]
\end{equation*}
denote a measurable version of the conditional oracle profile. Once the problem is reduced to rank space, $m_x$ is the only prompt-specific object that matters.

\subsubsection{Adaptive truncation among monotone reweightings}
\label{app:truncation_monotone_reweighting}

We first isolate the support-selection question of which proxy ranks should receive positive mass at all.

\begin{lemma}[Exact reduction to proxy-rank space]
\label{lem:oracle_rank_reduction}
For every measurable $g$ as above,
\begin{equation*}
\mathcal U_x(g)
=
\frac{\mathbb E_{\pi_{\mathrm{ref}}}[m_x(W_r)g(W_r)]}
{\mathbb E_{\pi_{\mathrm{ref}}}[g(W_r)]}.
\end{equation*}
If moreover the law of $r(x,Y)$ under $Y \sim \pi_{\mathrm{ref}}(\cdot \mid x)$ is continuous, then $W_r \sim \mathrm{Unif}[0,1]$, and therefore
\begin{equation*}
\mathcal U_x(g)
=
\frac{\int_0^1 m_x(w)g(w)\,dw}
{\int_0^1 g(w)\,dw}.
\end{equation*}
\end{lemma}

\begin{proof}
By definition of $\pi_g$,
\begin{equation*}
\mathcal U_x(g)
=
\frac{\mathbb E_{\pi_{\mathrm{ref}}}[W_u g(W_r)]}
{\mathbb E_{\pi_{\mathrm{ref}}}[g(W_r)]}.
\end{equation*}
Applying conditional expectation with respect to $W_r$ gives
\begin{equation*}
\mathbb E_{\pi_{\mathrm{ref}}}[W_u g(W_r)]
=
\mathbb E_{\pi_{\mathrm{ref}}}\!\left[\mathbb E[W_u g(W_r)\mid W_r]\right]
=
\mathbb E_{\pi_{\mathrm{ref}}}\!\left[g(W_r)\mathbb E[W_u\mid W_r]\right]
=
\mathbb E_{\pi_{\mathrm{ref}}}[m_x(W_r)g(W_r)].
\end{equation*}
This proves the first identity. If the law of $r(x,Y)$ is continuous, the probability integral transform yields $W_r \sim \mathrm{Unif}[0,1]$, and the second identity follows by writing expectations as Lebesgue integrals on $[0,1]$.
\end{proof}

Under the continuity assumption, the problem becomes a pure rank-space optimization over densities on $[0,1]$. Let
\begin{equation*}
\mathcal F_\uparrow
:=
\left\{
f:[0,1]\to[0,\infty):
\int_0^1 f(w)\,dw = 1,\ \ f \text{ is nondecreasing}
\right\},
\end{equation*}
and for each threshold $\lambda \in [0,1)$, define the lower-tail truncation density
\begin{equation*}
\nu_\lambda(w) := \frac{1}{1-\lambda}\mathds{1}\{w \in [\lambda,1]\}.
\end{equation*}

The next lemma shows that every monotone density is a mixture of these top-interval uniforms. This is the structural reason hard truncation is enough.

\begin{lemma}[Mixture representation of nondecreasing densities]
\label{lem:mixture_monotone}
Let $f \in \mathcal F_\uparrow$. Then there exists a probability measure $\xi$ on $[0,1)$ such that
\begin{equation*}
f(w)=\int_{[0,1)} \nu_\lambda(w)\,\xi(d\lambda)
\qquad
\text{for a.e. } w\in[0,1].
\end{equation*}
\end{lemma}

\begin{proof}
Define a finite Stieltjes measure on $[0,1)$ by
\begin{equation*}
\xi := f(0)\delta_0 + (1-t)\,df(t).
\end{equation*}
Since $f$ is nondecreasing and $\int_0^1 f(w)\,dw = 1$, integration by parts gives
\begin{equation*}
\xi([0,1))
=
f(0) + \int_{(0,1)} (1-t)\,df(t)
=
1.
\end{equation*}
Thus $\xi$ is a probability measure. For a.e. $w \in [0,1]$,
\begin{equation*}
\int_{[0,1)} \nu_\lambda(w)\,\xi(d\lambda)
=
\int_{[0,w]} \frac{1}{1-\lambda}\,\xi(d\lambda)
=
f(0) + \int_{(0,w]} df(\lambda)
=
f(w).
\end{equation*}
\end{proof}

\begin{theorem}[Hard truncation is optimal among monotone reweightings]
\label{thm:best_monotone_is_truncation}
Assume the law of $r(x,Y)$ under $Y \sim \pi_{\mathrm{ref}}(\cdot \mid x)$ is continuous. Then
\begin{equation*}
\sup_{f \in \mathcal F_\uparrow} \mathcal U_x(f)
=
\sup_{\lambda \in [0,1)} \mathcal U_x(\nu_\lambda).
\end{equation*}
Equivalently, in the oracle rank-space view, optimizing over monotone proxy-rank reweightings reduces to optimizing over hard lower-tail truncation rules.
\end{theorem}

\begin{proof}
By Lemma~\ref{lem:oracle_rank_reduction}, the continuity assumption implies
\begin{equation*}
\mathcal U_x(f)=\int_0^1 m_x(w)f(w)\,dw
\qquad
\text{for every } f \in \mathcal F_\uparrow.
\end{equation*}
Fix $f \in \mathcal F_\uparrow$. By Lemma~\ref{lem:mixture_monotone},
\begin{equation*}
f(w)=\int_{[0,1)} \nu_\lambda(w)\,\xi(d\lambda)
\qquad
\text{for a.e. } w,
\end{equation*}
for some probability measure $\xi$ on $[0,1)$. Therefore,
\begin{equation*}
\mathcal U_x(f)
=
\int_0^1 m_x(w)\left(\int_{[0,1)} \nu_\lambda(w)\,\xi(d\lambda)\right)dw.
\end{equation*}
By Fubini's theorem,
\begin{equation*}
\mathcal U_x(f)
=
\int_{[0,1)}\left(\int_0^1 m_x(w)\nu_\lambda(w)\,dw\right)\xi(d\lambda)
\le
\sup_{\lambda \in [0,1)} \mathcal U_x(\nu_\lambda).
\end{equation*}
Taking the supremum over $f \in \mathcal F_\uparrow$ gives
\begin{equation*}
\sup_{f \in \mathcal F_\uparrow}\mathcal U_x(f)
\le
\sup_{\lambda \in [0,1)} \mathcal U_x(\nu_\lambda).
\end{equation*}
The reverse inequality is immediate because each $\nu_\lambda$ belongs to $\mathcal F_\uparrow$.
\end{proof}

Theorem~\ref{thm:best_monotone_is_truncation} is an oracle statement, so it does not identify a single global threshold for practice. Its role is structural: in rank space, support restriction is the right family for the first decision, namely which proxy ranks to keep.

\begin{corollary}[QRPO comparison in the same rank-space model]
\label{cor:qrpo_rank_space}
Under the assumptions of Theorem~\ref{thm:best_monotone_is_truncation}, let
\begin{equation*}
q_\tau(w) := \frac{\tau e^{\tau w}}{e^\tau - 1},
\qquad \tau > 0,
\end{equation*}
with $q_0(w) \equiv 1$. Then
\begin{equation*}
\sup_{\tau \ge 0} \mathcal U_x(q_\tau)
\le
\sup_{\lambda \in [0,1)} \mathcal U_x(\nu_\lambda).
\end{equation*}
\end{corollary}

\begin{proof}
Each $q_\tau$ is nondecreasing on $[0,1]$, so $\{q_\tau : \tau \ge 0\} \subseteq \mathcal F_\uparrow$. The claim follows immediately from Theorem~\ref{thm:best_monotone_is_truncation}.
\end{proof}

Thus, within the same oracle rank-space model, the best hard truncation rule weakly dominates the best QRPO-style full-support monotone tilt.

\subsubsection{Within-tail reweighting after truncation}
\label{app:general_within_tail_upweighting}

We now hold the truncation level fixed and show that, when oracle value remains positively aligned with proxy rank inside the retained tail, uniform weighting over the retained interval is locally suboptimal, so a finite within-tail tilt can improve over pure truncation at the same threshold.

For a fixed threshold $\lambda \in (0,1)$ and parameter $\beta \in (0,\infty)$, define the truncated-soft family
\begin{equation*}
f_{\lambda,\beta}(w)
:=
\frac{1}{Z_\lambda(\beta)}
\left(
\frac{w-\lambda}{1-w+\lambda}
\right)^{1/\beta}
\mathds{1}\{w \in (\lambda,1)\},
\end{equation*}
where
\begin{equation*}
Z_\lambda(\beta)
:=
\int_\lambda^1
\left(
\frac{w-\lambda}{1-w+\lambda}
\right)^{1/\beta} dw.
\end{equation*}
This is the rank-space form induced by the shifted-truncated odds-ratio tilt at threshold $\lambda$. We use the convention $f_{\lambda,\infty} = \nu_\lambda$, so $\beta = \infty$ recovers hard truncation, while finite $\beta$ tilts mass toward larger retained proxy ranks without changing the support. Define
\begin{equation*}
\mathcal U_x(\lambda,\beta)
:=
\int_\lambda^1 m_x(w)f_{\lambda,\beta}(w)\,dw,
\qquad
\mathcal U_x(\lambda,\infty)
:=
\mathcal U_x(\nu_\lambda).
\end{equation*}

\begin{proposition}[Local criterion for improving over pure truncation]
\label{prop:local_upweighting}
Assume the law of $r(x,Y)$ under $Y \sim \pi_{\mathrm{ref}}(\cdot \mid x)$ is continuous, and fix $\lambda \in (0,1)$. Let
\begin{equation*}
h_\lambda(w)
:=
\log\!\left(\frac{w-\lambda}{1-w+\lambda}\right),
\qquad
w \in (\lambda,1).
\end{equation*}
Then
\begin{equation*}
\left.\frac{\partial}{\partial (\beta^{-1})}\mathcal U_x(\lambda,\beta)\right|_{\beta=\infty}
=
\operatorname{Cov}_{W_r \sim \mathrm{Unif}[\lambda,1]}
\!\bigl(m_x(W_r),\,h_\lambda(W_r)\bigr).
\end{equation*}
In particular, if this covariance is positive, then there exists a finite $\beta$ such that
\begin{equation*}
\mathcal U_x(\lambda,\beta) > \mathcal U_x(\lambda,\infty)=\mathcal U_x(\nu_\lambda).
\end{equation*}
\end{proposition}

\begin{proof}
Let $f_{\lambda,\infty} = \nu_\lambda$ denote the uniform density on $[\lambda,1]$. Then
\begin{equation*}
f_{\lambda,\beta}(w)
=
\frac{e^{h_\lambda(w)/\beta} f_{\lambda,\infty}(w)}
{\mathbb E_{W_r \sim \mathrm{Unif}[\lambda,1]}[e^{h_\lambda(W_r)/\beta}]}.
\end{equation*}
Therefore,
\begin{equation*}
\mathcal U_x(\lambda,\beta)
=
\frac{
\mathbb E_{W_r \sim \mathrm{Unif}[\lambda,1]}
\!\left[m_x(W_r)e^{h_\lambda(W_r)/\beta}\right]
}{
\mathbb E_{W_r \sim \mathrm{Unif}[\lambda,1]}
\!\left[e^{h_\lambda(W_r)/\beta}\right]
}.
\end{equation*}
Differentiating with respect to $\beta^{-1}$ at $\beta=\infty$ gives 
\begin{equation*}
\left.\frac{\partial}{\partial (\beta^{-1})}\mathcal U_x(\lambda,\beta)\right|_{\beta=\infty}
=
\mathbb E[m_x(W_r)h_\lambda(W_r)]
-
\mathbb E[m_x(W_r)]\,\mathbb E[h_\lambda(W_r)],
\end{equation*}
where all expectations are under $W_r \sim \mathrm{Unif}[\lambda,1]$. The differentiation may be passed under the expectation because $h_\lambda$ is integrable on $(\lambda,1)$ under the uniform law. This is exactly
\begin{equation*}
\operatorname{Cov}_{W_r \sim \mathrm{Unif}[\lambda,1]}
\!\bigl(m_x(W_r),\,h_\lambda(W_r)\bigr).
\end{equation*}
If this covariance is positive, then the derivative at $\beta=\infty$ in the $\beta^{-1}$ direction is positive, so by continuity there exists a finite $\beta$ such that
\begin{equation*}
\mathcal U_x(\lambda,\beta) > \mathcal U_x(\lambda,\infty).
\end{equation*}
\end{proof}

This is a fixed-$\lambda$ refinement statement, not a global dominance claim. It does not say that finite-$\beta$ reweighting beats the best threshold when $\lambda$ is optimized freely. Instead, once a global quality floor has been chosen, the covariance criterion identifies when a softer within-tail tilt extracts additional oracle value from the retained tail.

\subsubsection{Quadratic illustration underlying Figure~\ref{fig:theory}}
\label{app:quadratic_worked_example}

This subsection gives a minimal analytic example that separates the two effects shown in Figure~\ref{fig:theory}: changing the retained support and reweighting points within a fixed retained tail. Consider the stylized oracle-utility shape
\begin{equation*}
m_s(w) := 1-(w-s)^2,
\qquad
s\in(0,1),
\end{equation*}
where $s$ is the proxy-rank location at which oracle value is maximized. We use $m_s$ only as an unnormalized analytic shape. It is not itself a realizable conditional oracle win-rate profile, since
\begin{equation*}
\int_0^1 m_s(w)\,dw=\frac{2}{3}+s-s^2,
\end{equation*}
which is generally not $1/2$. A normalized profile with the same shape is
\begin{equation*}
\bar m_s(w)
:=
\frac12+\frac12\left(m_s(w)-\left(\frac23+s-s^2\right)\right).
\end{equation*}
Then $\int_0^1 \bar m_s(w)\,dw=1/2$, and $\bar m_s(w)\in[1/6,2/3]$ for all $s,w\in[0,1]$. For any density $f$ on $[0,1]$,
\begin{equation*}
\bar{\mathcal U}_s(f)
:=
\int_0^1 \bar m_s(w)f(w)\,dw
=
\frac12+\frac12\left(\mathcal U_s(f)-\left(\frac23+s-s^2\right)\right),
\end{equation*}
where
\begin{equation*}
\mathcal U_s(f):=\int_0^1 m_s(w)f(w)\,dw.
\end{equation*}
Thus, for each fixed $s$, $\bar{\mathcal U}_s$ is a positive affine transformation of $\mathcal U_s$. All maximizers, strict comparisons, crossing locations, and the threshold $s_\beta$ below are unchanged by this normalization. We therefore use the shorter unnormalized form $m_s$ in the algebra.

For the QRPO-style retained-tail family, write
\begin{equation*}
\mathcal U_s(\lambda,\beta)
:=
\int_\lambda^1 m_s(w)f_{\lambda,\beta}(w)\,dw,
\qquad
\mathcal U_s(\lambda,\infty)
:=
\mathcal U_s(\nu_\lambda).
\end{equation*}
The first proposition isolates the value of adaptive support restriction by comparing the best hard truncation rule with the best full-support QRPO tilt. The second proposition fixes $\lambda=1/2$, matching Figure~\ref{fig:theory}, and asks when finite within-tail reweighting improves over pure truncation at that same threshold.

\begin{proposition}[Adaptive truncation beats best QRPO under the quadratic profile]
\label{prop:quadratic_example}
For every $s>1/2$, the hard truncation threshold
\begin{equation*}
\lambda^\star(s):=\frac{3s-1}{2}
\end{equation*}
satisfies
\begin{equation*}
\mathcal U_s(\nu_{\lambda^\star(s)})
=
1-\frac{(1-s)^2}{4}
>
\sup_{\tau\ge 0}\mathcal U_s(q_\tau).
\end{equation*}
\end{proposition}

\begin{proof}
Since $m_s(w)=1-(w-s)^2$, maximizing $\mathcal U_s(f)$ is equivalent to minimizing
\begin{equation*}
R(f;s):=\mathbb E_{W_r\sim f}\!\left[(W_r-s)^2\right].
\end{equation*}
For the top-interval uniform law $\nu_\lambda$, we have
\begin{equation*}
\mathbb E_{\nu_\lambda}[W_r]=\frac{1+\lambda}{2},
\qquad
\operatorname{Var}_{\nu_\lambda}(W_r)=\frac{(1-\lambda)^2}{12}.
\end{equation*}
Therefore
\begin{equation*}
R(\nu_\lambda;s)
=
\left(\frac{1+\lambda}{2}-s\right)^2+\frac{(1-\lambda)^2}{12}.
\end{equation*}
Differentiating with respect to $\lambda$ gives
\begin{equation*}
\frac{d}{d\lambda}R(\nu_\lambda;s)=\frac{2\lambda}{3}+\frac{1}{3}-s.
\end{equation*}
Thus the unique minimizer over $\lambda\in[0,1]$ is
\begin{equation*}
\lambda^\star(s)=\frac{3s-1}{2},
\end{equation*}
which lies in $(0,1)$ for $s\in(1/2,1)$. Substituting this value gives
\begin{equation*}
R(\nu_{\lambda^\star(s)};s)=\frac{(1-s)^2}{4},
\qquad
\mathcal U_s(\nu_{\lambda^\star(s)})=1-\frac{(1-s)^2}{4}.
\end{equation*}

By Lemma~\ref{lem:mixture_monotone}, every $f\in\mathcal F_\uparrow$ is a mixture of the top-interval uniform laws $\nu_\lambda$. Since $\lambda\mapsto\mathcal U_s(\nu_\lambda)$ has the unique maximizer $\lambda^\star(s)$, the unique maximizer over $\mathcal F_\uparrow$ is $\nu_{\lambda^\star(s)}$. Hence
\begin{equation*}
\sup_{\tau\ge0}\mathcal U_s(q_\tau)
\le
\mathcal U_s(\nu_{\lambda^\star(s)}).
\end{equation*}

It remains only to rule out equality. No finite-$\tau$ QRPO density equals $\nu_{\lambda^\star(s)}$: $q_0$ is uniform on $[0,1]$, while $q_\tau$ is strictly positive on all of $[0,1]$ for every $\tau>0$, whereas $\nu_{\lambda^\star(s)}$ has zero density below $\lambda^\star(s)>0$. Moreover, since $m_s$ is bounded and continuous and $q_\tau$ converges weakly to a point mass at $w=1$ as $\tau\to\infty$,
\begin{equation*}
\lim_{\tau\to\infty}\mathcal U_s(q_\tau)=m_s(1)=1-(1-s)^2.
\end{equation*}
This limiting value is strictly smaller than
\begin{equation*}
1-\frac{(1-s)^2}{4}
=
\mathcal U_s(\nu_{\lambda^\star(s)}).
\end{equation*}
Therefore neither a finite QRPO tilt nor its limiting point mass at $w=1$ attains the hard-truncation optimum, and the desired strict inequality follows.
\end{proof}

This comparison isolates the gain from support restriction itself. Even in the quadratic example, a full-support QRPO tilt cannot recover the oracle value achieved by the best hard truncation rule.

\begin{proposition}[At $\lambda=1/2$, finite $\beta$ helps above an explicit threshold]
\label{prop:fixed_half_threshold}
Define
\begin{equation*}
s_\beta
:=
\frac{\frac{11}{12}-\log 2}{1-\log 2}
\approx 0.7284.
\end{equation*}
Then
\begin{equation*}
\left.\frac{\partial}{\partial(\beta^{-1})}\mathcal U_s(1/2,\beta)\right|_{\beta=\infty}
=
(1-\log 2)s-\left(\frac{11}{12}-\log 2\right).
\end{equation*}
In particular, if $s>s_\beta$, then there exists a finite $\beta$ such that
\begin{equation*}
\mathcal U_s(1/2,\beta)>\mathcal U_s(1/2,\infty)=\mathcal U_s(\nu_{1/2}).
\end{equation*}
\end{proposition}

\begin{proof}
We instantiate Proposition~\ref{prop:local_upweighting} at the fixed threshold $\lambda=1/2$. On the retained interval $[1/2,1]$, the local log-odds tilt is
\begin{equation*}
h(w)
:=
\log\!\left(\frac{w-\frac12}{\frac32-w}\right),
\qquad
w\in(1/2,1).
\end{equation*}
Proposition~\ref{prop:local_upweighting} gives
\begin{equation*}
\left.\frac{\partial}{\partial(\beta^{-1})}\mathcal U_s(1/2,\beta)\right|_{\beta=\infty}
=
\operatorname{Cov}_{W_r\sim\operatorname{Unif}[1/2,1]}
\!\bigl(m_s(W_r),h(W_r)\bigr).
\end{equation*}
Since $m_s(W_r)=1-(W_r-s)^2$, this is
\begin{equation*}
\left.\frac{\partial}{\partial(\beta^{-1})}\mathcal U_s(1/2,\beta)\right|_{\beta=\infty}
=
-\operatorname{Cov}_{W_r\sim\operatorname{Unif}[1/2,1]}
\!\bigl((W_r-s)^2,h(W_r)\bigr).
\end{equation*}

The required covariance is explicit. Under $W_r\sim\operatorname{Unif}[1/2,1]$,
\begin{equation*}
\mathbb E[h(W_r)]
=
2\int_{1/2}^1
\log\!\left(\frac{w-\frac12}{\frac32-w}\right)dw
=
-\log 4,
\end{equation*}
and
\begin{equation*}
\mathbb E[(W_r-s)^2]
=
\left(s-\frac34\right)^2+\frac{1}{48}
=
s^2-\frac32s+\frac{7}{12}.
\end{equation*}
A direct integration also gives
\begin{equation*}
\mathbb E[(W_r-s)^2h(W_r)]
=
\frac{11}{12}-s+\left(-2s^2+4s-\frac{13}{6}\right)\log 2.
\end{equation*}
Subtracting the product of the first two moments yields
\begin{equation*}
\operatorname{Cov}\!\bigl((W_r-s)^2,h(W_r)\bigr)
=
\frac{11}{12}-\log 2-(1-\log 2)s.
\end{equation*}
Therefore
\begin{equation*}
\left.\frac{\partial}{\partial(\beta^{-1})}\mathcal U_s(1/2,\beta)\right|_{\beta=\infty}
=
(1-\log 2)s-\left(\frac{11}{12}-\log 2\right).
\end{equation*}
The threshold $s_\beta$ is the unique zero of this affine function. Hence, for every $s>s_\beta$, the derivative in the $\beta^{-1}$ direction is positive at $\beta=\infty$, so for sufficiently small positive $\beta^{-1}$, equivalently for sufficiently large finite $\beta$, we have
\begin{equation*}
\mathcal U_s(1/2,\beta)>\mathcal U_s(1/2,\infty).
\end{equation*}
\end{proof}

\begin{figure}[tp]
     \centering
     % First Subfigure
     \begin{subfigure}[t]{0.50\textwidth}
         \centering
         \includegraphics[width=\textwidth]{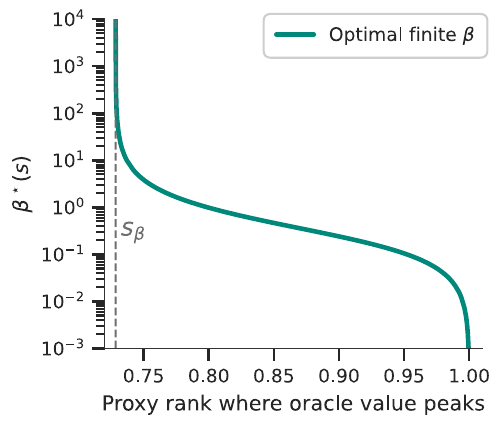}
     \end{subfigure}
    \caption{Optimal finite sharpness in the fixed-threshold quadratic illustration. For $\lambda=1/2$ and $m_s(w)=1-(w-s)^2$, we numerically maximize $\mathcal U_s(1/2,\beta)$ over finite $\beta$ for each oracle peak location $s>s_\beta$. The dashed vertical line marks $s_\beta\approx0.7284$, below which pure truncation is locally optimal and the optimizer is $\beta=\infty$. Just above $s_\beta$, the optimal finite tilt is arbitrarily weak, so $\beta^\star(s)$ is large; as $s$ approaches one, the optimal policy concentrates more strongly near the top proxy ranks, and $\beta^\star(s)$ decreases.}

    \label{fig:optimal_beta_quadratic}
\end{figure}

The proposition identifies the point at which pure truncation becomes locally suboptimal, but it does not assign a single universal value of $\beta$. The optimal value depends on the oracle peak location $s$. To compute it, write $\alpha=1/\beta$ and view $\mathcal U_s(1/2,\beta)$ as a one-dimensional function of $\alpha\ge 0$. Differentiating the tilted expectation gives
\begin{equation*}
\frac{\partial}{\partial \alpha}\mathcal U_s(1/2,\beta)
=
\operatorname{Cov}_{f_{1/2,\beta}}
\left(
m_s(W),
\log\left(\frac{W-\frac12}{\frac32-W}\right)
\right).
\end{equation*}
Thus an interior optimum $\beta^\star(s)$ is obtained by solving this scalar covariance equation, while for $s\le s_\beta$ the optimum remains the boundary value $\beta=\infty$, corresponding to pure truncation. Figure~\ref{fig:optimal_beta_quadratic} plots the resulting finite optimizer for $s>s_\beta$: $\beta^\star(s)$ diverges as $s$ approaches $s_\beta$ from above, because the beneficial tilt is infinitesimal at the threshold, and decreases as $s$ approaches one, where stronger concentration near the top proxy ranks becomes optimal. When the proxy is highly reliable $s\to 1$ and thus $\beta^\star\to0$.

\section{Additional experiments and results}
\label{appx:additional_exps}

% \subsection{Mistral-7B experiments}
% \label{appx:mistral}
% In this scenario, we ran a supervised-fine-tuning on the base model (Mistral-7B-Instruct-v0.2). 

\subsection{Raw rewards}
\label{appx:raw_rewards}
\begin{table}[thp]
\centering
\caption{Raw reward for Llama 8B Tülu 3 SFT on UltraFeedback eval split (UF) and AlpacaEval (AE), with reporting token length.}
\label{tab:llama_ultrafeedback_raw_reward}
\resizebox{\textwidth}{!}{%
\begin{tabular}{l cc cc cc c}
\toprule
& \multicolumn{2}{c}{ArmoRM} & \multicolumn{2}{c}{Skywork-v2-Llama} & \multicolumn{2}{c}{Skywork-v2-Qwen} & \multicolumn{1}{c}{UF} \\
\cmidrule(lr){2-3} \cmidrule(lr){4-5} \cmidrule(lr){6-7} \cmidrule(lr){8-8}
Method & UF & AE & UF & AE & UF & AE & Len \\
\midrule
Initial & $0.1300{\scriptstyle\pm 0.0008}$ & $0.1370{\scriptstyle\pm 0.0003}$ & $9.15{\scriptstyle\pm 0.18}$ & $10.78{\scriptstyle\pm 0.05}$ & $3.78{\scriptstyle\pm 0.15}$ & $4.97{\scriptstyle\pm 0.03}$ & $389.1{\scriptstyle\pm 8.0}$ \\
\midrule
DPO & $0.1497{\scriptstyle\pm 0.0002}$ & $0.1620{\scriptstyle\pm 0.0003}$ & $16.27{\scriptstyle\pm 0.10}$ & $21.36{\scriptstyle\pm 0.18}$ & $8.00{\scriptstyle\pm 0.07}$ & $10.87{\scriptstyle\pm 0.07}$ & $501.5{\scriptstyle\pm 2.3}$ \\
REBEL & $0.1489{\scriptstyle\pm 0.0002}$ & $0.1620{\scriptstyle\pm 0.0002}$ & $16.35{\scriptstyle\pm 0.07}$ & $22.15{\scriptstyle\pm 0.16}$ & $8.00{\scriptstyle\pm 0.06}$ & $11.20{\scriptstyle\pm 0.04}$ & $516.1{\scriptstyle\pm 3.6}$ \\
REBEL (rand) & $0.1436{\scriptstyle\pm 0.0002}$ & $0.1568{\scriptstyle\pm 0.0003}$ & $14.28{\scriptstyle\pm 0.10}$ & $20.05{\scriptstyle\pm 0.09}$ & $6.81{\scriptstyle\pm 0.02}$ & $10.18{\scriptstyle\pm 0.03}$ & $473.9{\scriptstyle\pm 5.2}$ \\
QRPO & $0.1517{\scriptstyle\pm 0.0007}$ & $0.1666{\scriptstyle\pm 0.0001}$ & $17.41{\scriptstyle\pm 0.07}$ & $21.08{\scriptstyle\pm 0.14}$ & $8.24{\scriptstyle\pm 0.01}$ & $10.31{\scriptstyle\pm 0.17}$ & $730.1{\scriptstyle\pm 4.8}$ \\
QRPO (rand) & $0.1465{\scriptstyle\pm 0.0001}$ & $0.1610{\scriptstyle\pm 0.0001}$ & $15.83{\scriptstyle\pm 0.09}$ & $21.87{\scriptstyle\pm 0.12}$ & $7.45{\scriptstyle\pm 0.06}$ & $10.91{\scriptstyle\pm 0.06}$ & $508.6{\scriptstyle\pm 2.8}$ \\
BoNBoN & $0.1498{\scriptstyle\pm 0.0005}$ & $0.1632{\scriptstyle\pm 0.0004}$ & $18.30{\scriptstyle\pm 0.06}$ & $24.29{\scriptstyle\pm 0.22}$ & $8.89{\scriptstyle\pm 0.02}$ & $12.25{\scriptstyle\pm 0.05}$ & $548.3{\scriptstyle\pm 3.0}$ \\
\midrule
\textbf{TUP (ours)} & $0.1531{\scriptstyle\pm 0.0001}$ & $0.1674{\scriptstyle\pm 0.0002}$ & $\mathbf{20.20{\scriptstyle\pm 0.07}}$ & $\mathbf{26.88{\scriptstyle\pm 0.09}}$ & $\mathbf{9.75{\scriptstyle\pm 0.02}}$ & $\mathbf{13.34{\scriptstyle\pm 0.02}}$ & $645.5{\scriptstyle\pm 2.7}$ \\
\bottomrule
\end{tabular}}
\end{table}
\begin{table}[t]
\centering
\caption{Raw reward for Llama 8B Tülu 3 SFT on Magpie Air eval split (MA) and AlpacaEval (AE), with reporting token length.}
\label{tab:llama_magpie_raw_reward}
\resizebox{\textwidth}{!}{%
\begin{tabular}{l cc cc cc c}
\toprule
& \multicolumn{2}{c}{ArmoRM} & \multicolumn{2}{c}{Skywork-v2-Llama} & \multicolumn{2}{c}{Skywork-v2-Qwen} & \multicolumn{1}{c}{MA} \\
\cmidrule(lr){2-3} \cmidrule(lr){4-5} \cmidrule(lr){6-7} \cmidrule(lr){8-8}
Method & MA & AE & MA & AE & MA & AE & Len \\
\midrule
Initial & $0.1530{\scriptstyle\pm 0.0010}$ & $0.1370{\scriptstyle\pm 0.0003}$ & $10.01{\scriptstyle\pm 0.36}$ & $10.78{\scriptstyle\pm 0.05}$ & $4.99{\scriptstyle\pm 0.29}$ & $4.97{\scriptstyle\pm 0.03}$ & $331.6{\scriptstyle\pm 2.4}$ \\
\midrule
DPO  & $0.1800{\scriptstyle\pm 0.0003}$ & $0.1581{\scriptstyle\pm 0.0006}$ & $19.23{\scriptstyle\pm 0.12}$ & $18.82{\scriptstyle\pm 0.19}$ & $11.93{\scriptstyle\pm 0.09}$ & $9.69{\scriptstyle\pm 0.08}$ & $484.5{\scriptstyle\pm 3.2}$ \\
REBEL & $0.1795{\scriptstyle\pm 0.0003}$ & $0.1591{\scriptstyle\pm 0.0007}$ & $20.69{\scriptstyle\pm 0.31}$ & $20.65{\scriptstyle\pm 0.19}$ & $12.31{\scriptstyle\pm 0.14}$ & $10.37{\scriptstyle\pm 0.12}$ & $543.2{\scriptstyle\pm 4.5}$ \\
REBEL (random) & $0.1778{\scriptstyle\pm 0.0002}$ & $0.1565{\scriptstyle\pm 0.0003}$ & $17.60{\scriptstyle\pm 0.20}$ & $17.81{\scriptstyle\pm 0.14}$ & $10.88{\scriptstyle\pm 0.14}$ & $9.14{\scriptstyle\pm 0.08}$ & $479.8{\scriptstyle\pm 4.4}$ \\
QRPO & $0.1772{\scriptstyle\pm 0.0003}$ & $0.1578{\scriptstyle\pm 0.0000}$ & $18.23{\scriptstyle\pm 0.23}$ & $18.99{\scriptstyle\pm 0.09}$ & $10.25{\scriptstyle\pm 0.13}$ & $9.30{\scriptstyle\pm 0.01}$ & $554.0{\scriptstyle\pm 2.2}$ \\
QRPO (random) & $0.1717{\scriptstyle\pm 0.0004}$ & $0.1549{\scriptstyle\pm 0.0007}$ & $16.46{\scriptstyle\pm 0.15}$ & $18.20{\scriptstyle\pm 0.26}$ & $9.28{\scriptstyle\pm 0.11}$ & $9.03{\scriptstyle\pm 0.15}$ & $473.2{\scriptstyle\pm 0.6}$ \\
BoNBoN & $0.1654{\scriptstyle\pm 0.0003}$ & $0.1453{\scriptstyle\pm 0.0008}$ & $14.80{\scriptstyle\pm 0.21}$ & $14.69{\scriptstyle\pm 0.12}$ & $8.30{\scriptstyle\pm 0.06}$ & $7.21{\scriptstyle\pm 0.13}$ & $416.2{\scriptstyle\pm 1.6}$ \\
\midrule
\textbf{TUP (ours)} & $0.1791{\scriptstyle\pm 0.0004}$ & $0.1583{\scriptstyle\pm 0.0006}$ & $22.32{\scriptstyle\pm 0.23}$ & $21.48{\scriptstyle\pm 0.25}$ & $13.12{\scriptstyle\pm 0.12}$ & $10.52{\scriptstyle\pm 0.03}$ & $700.9{\scriptstyle\pm 4.3}$ \\
\bottomrule
\end{tabular}}
\end{table}
\begin{table}[t]
\centering
\caption{Raw reward for Mistral 7B SFT on Magpie Air eval split (MA) and AlpacaEval (AE), with reporting token length.}
\label{tab:mistral_sft_magpie_raw_reward}
\resizebox{\textwidth}{!}{%
\begin{tabular}{l cc cc cc c}
\toprule
& \multicolumn{2}{c}{ArmoRM} & \multicolumn{2}{c}{Skywork-v2-Llama} & \multicolumn{2}{c}{Skywork-v2-Qwen} & \multicolumn{1}{c}{MA} \\
\cmidrule(lr){2-3} \cmidrule(lr){4-5} \cmidrule(lr){6-7} \cmidrule(lr){8-8}
Method & MA & AE & MA & AE & MA & AE & Len \\
\midrule
Initial & $0.1628{\scriptstyle\pm 0.0002}$ & $0.1478{\scriptstyle\pm 0.0002}$ & $15.99{\scriptstyle\pm 0.15}$ & $16.55{\scriptstyle\pm 0.20}$ & $9.86{\scriptstyle\pm 0.08}$ & $8.77{\scriptstyle\pm 0.09}$ & $493.1{\scriptstyle\pm 4.1}$ \\
\midrule
DPO  & $0.1743{\scriptstyle\pm 0.0003}$ & $0.1530{\scriptstyle\pm 0.0005}$ & $22.71{\scriptstyle\pm 0.28}$ & $20.64{\scriptstyle\pm 0.23}$ & $16.75{\scriptstyle\pm 0.19}$ & $11.84{\scriptstyle\pm 0.13}$ & $750.1{\scriptstyle\pm 2.7}$ \\
REBEL & $0.1593{\scriptstyle\pm 0.0004}$ & $0.1424{\scriptstyle\pm 0.0014}$ & $14.27{\scriptstyle\pm 0.20}$ & $11.68{\scriptstyle\pm 0.20}$ & $10.18{\scriptstyle\pm 0.10}$ & $6.50{\scriptstyle\pm 0.13}$ & $1375.7{\scriptstyle\pm 6.8}$ \\
REBEL (random) & $0.1717{\scriptstyle\pm 0.0001}$ & $0.1539{\scriptstyle\pm 0.0002}$ & $\mathbf{23.29{\scriptstyle\pm 0.05}}$ & $\mathbf{20.97{\scriptstyle\pm 0.04}}$ & $\mathbf{17.19{\scriptstyle\pm 0.10}}$ & $\mathbf{12.08{\scriptstyle\pm 0.01}}$ & $861.0{\scriptstyle\pm 6.5}$ \\
QRPO & $0.1710{\scriptstyle\pm 0.0006}$ & $0.1532{\scriptstyle\pm 0.0002}$ & $21.14{\scriptstyle\pm 0.13}$ & $19.80{\scriptstyle\pm 0.06}$ & $14.28{\scriptstyle\pm 0.03}$ & $10.86{\scriptstyle\pm 0.11}$ & $629.1{\scriptstyle\pm 2.2}$ \\
QRPO (random) & $0.1522{\scriptstyle\pm 0.0005}$ & $0.1301{\scriptstyle\pm 0.0005}$ & $18.39{\scriptstyle\pm 0.12}$ & $15.29{\scriptstyle\pm 0.13}$ & $12.93{\scriptstyle\pm 0.03}$ & $8.27{\scriptstyle\pm 0.11}$ & $1155.9{\scriptstyle\pm 6.2}$ \\
BoNBoN & $0.1655{\scriptstyle\pm 0.0003}$ & $0.1491{\scriptstyle\pm 0.0006}$ & $17.33{\scriptstyle\pm 0.26}$ & $17.42{\scriptstyle\pm 0.13}$ & $11.07{\scriptstyle\pm 0.14}$ & $9.32{\scriptstyle\pm 0.09}$ & $516.4{\scriptstyle\pm 1.7}$ \\
\midrule
\textbf{TUP mild} & $0.1757{\scriptstyle\pm 0.0004}$ & $0.1537{\scriptstyle\pm 0.0004}$ & $22.38{\scriptstyle\pm 0.03}$ & $19.94{\scriptstyle\pm 0.10}$ & $16.02{\scriptstyle\pm 0.01}$ & $11.09{\scriptstyle\pm 0.05}$ & $808.2{\scriptstyle\pm 10.1}$ \\
\textbf{TUP mid.} & $\mathbf{0.1767{\scriptstyle\pm 0.0003}}$ & $0.1562{\scriptstyle\pm 0.0002}$ & $22.21{\scriptstyle\pm 0.18}$ & $20.23{\scriptstyle\pm 0.21}$ & $16.45{\scriptstyle\pm 0.14}$ & $11.52{\scriptstyle\pm 0.07}$ & $731.2{\scriptstyle\pm 3.9}$ \\
\textbf{TUP aggressive} & $0.1749{\scriptstyle\pm 0.0005}$ & $\mathbf{0.1564{\scriptstyle\pm 0.0006}}$ & $20.47{\scriptstyle\pm 0.07}$ & $19.33{\scriptstyle\pm 0.16}$ & $15.31{\scriptstyle\pm 0.05}$ & $10.91{\scriptstyle\pm 0.07}$ & $744.4{\scriptstyle\pm 1.4}$ \\
\bottomrule
\end{tabular}}
\end{table}

Tables~\ref{tab:llama_ultrafeedback_raw_reward} and~\ref{tab:llama_magpie_raw_reward} report the corresponding raw reward scores and generated lengths. On UltraFeedback, TUP remains strongest under both independent Skywork reward models, including on AlpacaEval. On Magpie Air, TUP remains competitive with the strongest baselines and improves over QRPO and BoNBoN under the independent Skywork evaluators. These raw results should be read together with the LC tables because several methods, including TUP and QRPO, tend to produce longer responses. 

\begin{figure}[htbp]

     \centering
     % First Subfigure

     \begin{subfigure}[t]{0.32\textwidth}
         \centering
         \includegraphics[width=\textwidth]{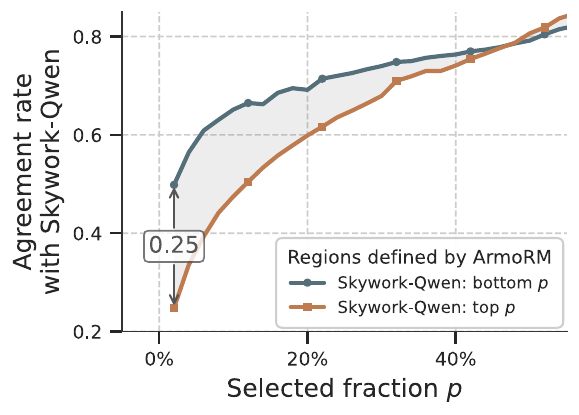}
     \end{subfigure}
     \hfill
     \begin{subfigure}[t]{0.32\textwidth}
         \centering
         \includegraphics[width=\textwidth]{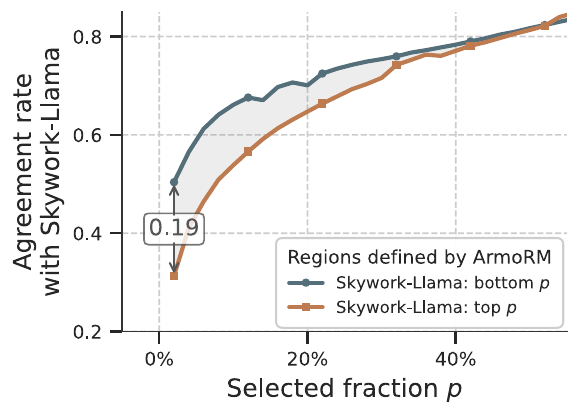}
     \end{subfigure}
     \hfill
     \begin{subfigure}[t]{0.32\textwidth}
         \centering
         \includegraphics[width=\textwidth]{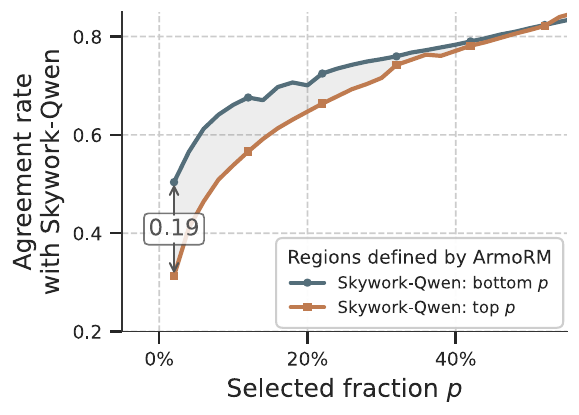}
     \end{subfigure}
     
     \caption{Additional cross-model agreement rates between ArmoRM and independent Skywork reward models. For each selected fraction $p$, the plots show the fraction of completions placed by ArmoRM in the top or bottom $p$ region that the Skywork evaluator also places in the corresponding region. \textbf{Left:} UltraFeedback with Skywork-v2-Qwen. \textbf{Center:} Magpie Air with Skywork-v2-Llama. \textbf{Right:} Magpie Air with Skywork-v2-Qwen.}
    \label{fig:top_bot_disagreement_additional}
\end{figure}

\begin{figure}[htbp]

     \centering
     % First Subfigure

     \begin{subfigure}[t]{0.32\textwidth}
         \centering
         \includegraphics[width=\textwidth]{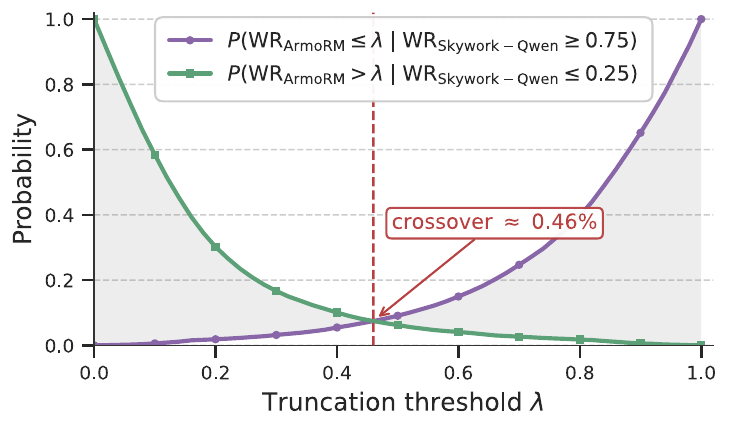}
     \end{subfigure}
     \hfill
     \begin{subfigure}[t]{0.32\textwidth}
         \centering
         \includegraphics[width=\textwidth]{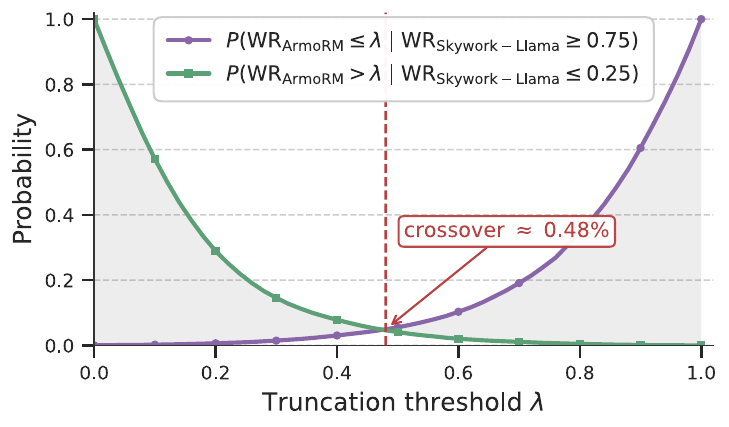}
     \end{subfigure}
     \hfill
     \begin{subfigure}[t]{0.32\textwidth}
         \centering
         \includegraphics[width=\textwidth]{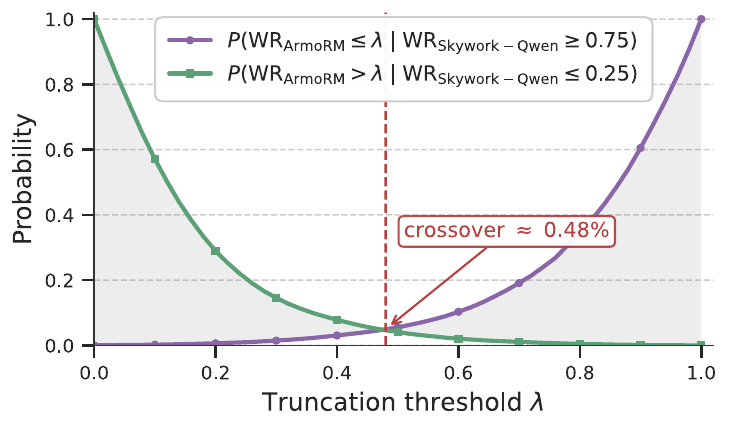}
     \end{subfigure}
     
     \caption{Additional cost--benefit curves for $\lambda$-truncation. Each plot compares the probability that truncation discards a completion ranked in the top $25\%$ by the independent Skywork evaluator with the probability that it retains a completion ranked in the bottom $25\%$ by that evaluator. \textbf{Left:} UltraFeedback with Skywork-v2-Qwen. \textbf{Center:} Magpie Air with Skywork-v2-Llama. \textbf{Right:} Magpie Air with Skywork-v2-Qwen.}

    \label{fig:tradeoff_additional}
\end{figure}

Figures~\ref{fig:top_bot_disagreement_additional} and~\ref{fig:tradeoff_additional} extend the motivating observation in Figure~\ref{fig:top_bot_disagreement} beyond the UltraFeedback/Skywork-v2-Llama setting. Across the additional dataset--evaluator combinations, ArmoRM-defined bottom regions are more consistently recovered by the independent Skywork evaluator than ArmoRM-defined top regions. The corresponding $\lambda$ curves show the same truncation tradeoff: increasing $\lambda$ reduces the chance of retaining completions that the independent evaluator ranks poorly, but also increases the chance of discarding completions it ranks highly. This supports using moderate truncation rather than an aggressive threshold.

We also estimate the effective sharpness learned by the trained TUP checkpoint by reversing the construction used during training. During training, the truncated win-rate label $w_\lambda$ determines the target policy-reference log-ratio through the configured sharpness $\beta$. After training, we instead observe the checkpoint log-ratio
\begin{equation*}
\Delta_\theta(y \mid x)
=
\log \pi_\theta(y \mid x)-\log \pi_{\mathrm{ref}}(y \mid x),
\end{equation*}
and ask which value of $\beta$ best explains its dependence on the reward-derived truncated win-rate. On the \texttt{test} split, we construct $w_\lambda$ by ranking each response reward against the 12 reference completions for the same prompt and applying the same threshold $\lambda=0.5$ used in training. We then fit the linear relation between $\Delta_\theta(y \mid x)$ and $z=\operatorname{logit}(w_\lambda)$ on the finite, unclipped subset. This gives an estimate of $\beta_{\mathrm{eff}}\approx 0.0147$. Thus, although training used the nominal value $\beta=0.01$, the learned policy behaves on this diagnostic as if its effective sharpness is roughly $1.5$ times larger.

\subsection{Length-matched rewards}
\label{appx:length-matched}
\begin{table}[thp]
\centering
\caption{TUP (with $\lambda=0.5$) win rates against each benchmark method on similar-length AlpacaEval responses for Llama 8B Tülu 3 SFT trained on Magpie Air. For each comparison, we retain response pairs whose lengths differ by at most $10\%$. Values above $50\%$ favor TUP and are shown in bold. The final column gives the number of retained pairs out of $2{,}415$.}
\label{tab:length-matched}
\resizebox{0.9\textwidth}{!}{%
\begin{tabular}{lcccc}
\toprule
Method & ArmoRM & Skywork-Llama & Skywork-Qwen & Prompts retained (of 2415) \\
\midrule
DPO & \textbf{53.1\%} & \textbf{55.3\%} & 49.1\% & 276 \\
REBEL & \textbf{50.9\%} & \textbf{50.2\%} & 45.8\% & 463 \\
REBEL (random) & \textbf{64.9\%} & \textbf{63.2\%} & \textbf{61.8\%} & 238 \\
QRPO & \textbf{51.4\%} & \textbf{56.6\%} & \textbf{56.6\%} & 415 \\
QRPO (random) & \textbf{54.2\%} & \textbf{53.1\%} & 46.7\% & 240 \\
BoNBoN & \textbf{71.0\%} & \textbf{60.7\%} & \textbf{56.5\%} & 169 \\
\midrule
Average & \textbf{57.6\%} & \textbf{56.5\%} & \textbf{52.7\%} & N/A \\
\bottomrule
\end{tabular}}
\end{table}

While the LC-reward results in Tables~\ref{tab:llama_indataset_lc_reward} and~\ref{tab:llama_alpacaeval} already account for response length at the metric level, we seek further evidence that verbosity is not the sole contributor to TUP's strong performance. To this end, Table~\ref{tab:length-matched} compares TUP with each baseline on responses of similar length. TUP is preferred over all benchmark methods under ArmoRM and Skywork-Llama. The results under Skywork-Qwen are more mixed across methods, although TUP's average win rate remains above $50\%$ under all three reward models. Importantly, the matching criterion retains only a fraction of the response pairs and produces a different subset for each baseline. These results should therefore be viewed as complementary evidence rather than as an estimate of performance on the full test set. Together with the LC-reward results, they further suggest that TUP's strong performance is not solely due to its longer responses.

\subsection{Global vs. per-prompt truncation}
\label{appx:global_vs_per_prompt}

To estimate the potential headroom of prompt-adaptive truncation over a fixed global threshold $\lambda$, we consider an idealized empirical setting in which Skywork-Llama serves as the ``oracle'' evaluator. For each UltraFeedback prompt $i$, let $(y_i^{(1)},y_i^{(2)},\ldots,y_i^{(N)})$ denote the responses ranked by ArmoRM. We then:

\begin{enumerate}
    \item Compute the ideal prompt-specific threshold
    \begin{equation*}
    \lambda_i^* =
    \frac{\argmax_j r_{\text{Skywork}}(y_i^{(j)})-1}{N},
    \end{equation*}
    which discards all responses ranked below the Skywork-best response by ArmoRM while retaining that response.
    \item Select the single global threshold $\lambda^*$ that maximizes the average Skywork-Llama win rate across all prompts.
\end{enumerate}

To isolate the effect of the truncation threshold, we apply no within-tail reweighting. Under this idealized setting, the prompt-specific thresholds achieve an empirical Skywork-based win rate of $0.9325$, compared with $0.8820$ for the best fixed global threshold. Thus, although prompt-specific truncation offers some additional headroom, the fixed global threshold captures most of the attainable win rate in this analysis.

\section{Implementation details}
\label{appx:imp}

\subsection{Codebase, models, and data}

The experiments in Tables~\ref{tab:llama_indataset_lc_reward}--\ref{tab:mistral_sft_magpie_raw_reward} use the released QRPO benchmark data rather than regenerating reference completions or reward annotations. These data contain prompt-specific reference-completion pools scored by ArmoRM.

All experiments are implemented using the official QRPO repository\footnote{\url{https://github.com/CLAIRE-Labo/quantile-reward-policy-optimization.git}}, together with its released preprocessed reference datasets\footnote{\url{https://huggingface.co/collections/skandermoalla/qrpo-reference-datasets}}. We use Llama 8B Tülu 3 SFT as the initial policy. On Mistral 7B we run a dedicated supervised fine-tuning (STF) on the original model (Mistral-7B-Instruct-v0.2), we closely followed QRPO code, using the prompt–chosen-response from Magpie-Air-DPO-100K-v0.1. After filtering sequences to at most 2,048 tokens, the training set contained 97,812 examples, with 1,992 held out for evaluation. We trained for one epoch in \texttt{bfloat16} using an effective batch size of 128 across four GPUs, a learning rate of \(5 \times 10^{-7}\), 10\% warm-up, and cosine decay, without LoRA or other parameter-efficient tuning methods. To reproduce, please refer to our or QRPO GitHub repository.

The training reward model in the released data is \texttt{RLHFlow/ArmoRM-Llama3-8B-v0.1}. For transfer evaluation, we rescore generated completions with \texttt{Skywork/Skywork-Reward-V2-Llama-3.1-8B} and \newline\texttt{Skywork/Skywork-Reward-V2-Qwen3-8B}. 

For AlpacaEval, we report the win rate and length-controlled win rate against the default \texttt{gpt4-turbo} reference outputs, which correspond to \texttt{gpt-4-1106-preview}; \texttt{gpt-4o} is used only as the judge due to the deprecation of \texttt{gpt-4-1106-preview}.

The UltraFeedback setting contains $61{,}024$ training prompts and a held-out set of $1{,}995$ prompts, split into $998$ checkpoint-selection prompts and $997$ reward-reporting prompts. The Magpie Air setting contains $97{,}812$ training prompts and a held-out set of $1{,}992$ prompts, split into $996$ checkpoint-selection prompts and $996$ reward-reporting prompts. AlpacaEval contains $805$ evaluation prompts. The QRPO benchmark reference pools contain $50$ reference completions per prompt.

For DPO, REBEL, and QRPO, we use the best-performing non-SFT, off-policy configurations reported in Tables 12 and 14 of~\citet{matrenok_quantile_2025}. We use the \texttt{offpolicy2best} and \texttt{offpolicy2random} variants as reported in Tables~\ref{tab:llama_ultrafeedback_hparams} and~\ref{tab:llama_magpie_hparams}. BoNBoN and TUP are not included in the original QRPO benchmark, so we perform a separate hyperparameter search for these methods.

\subsection{Reward models used for evaluation}
\label{appx:reward_model_leaderboard}

Table~\ref{tab:rewardbench} summarizes the external reward models used in our experiments. ArmoRM is used only as the offline proxy reward model for training labels, whereas the Skywork-v2 models and \texttt{gpt-4o} are used only for evaluation.  The RewardBench ranks are provided only as external context, since leaderboard positions can change as new models are added.

\begin{table}[t]
\centering
\caption{RewardBench standings (regardless of public availability) for the reward and judge models used in our experiments, retrieved in April 2026. Global Rank denotes the model's overall position on the RewardBench leaderboard, and Type Rank denotes its position within the corresponding model category. The asterisk marks \texttt{gpt-4o}; because AlpacaEval uses an API-hosted judge, the exact served model version may not match the RewardBench entry.}
\label{tab:rewardbench}
\vspace{0.8em}
\resizebox{0.95\textwidth}{!}{%
\begin{tabular}{l l l c c}
\toprule
Model & Role in this work & Model type & Global Rank & Type Rank \\
\midrule
ArmoRM-Llama3-8B-v0.1 & Training proxy reward model & Classifier & \#34 & \#26 \\
Skywork-Reward-V2-Llama-3.1-8B & Evaluation reward model & Classifier & \#1 & \#1 \\
Skywork-Reward-V2-Qwen3-8B & Evaluation reward model & Classifier & \#6 & \#3 \\
$\texttt{gpt-4o}^*$ & GPT judge & Generative & \#36 & \#10 \\
\bottomrule
\end{tabular}}
\end{table}

\subsection{TUP implementation}

TUP uses the same precomputed completion pools and ArmoRM scores as the QRPO benchmark. For each prompt, we convert the ArmoRM scores in the pool into empirical in-pool win-rates, apply the shifted truncation $\hat{w}_{\lambda,r}=\max(\hat{w}_r-\lambda,0)$, and train with the BCE objective in Algorithm~\ref{alg:bce_alignment}. Relative to the QRPO pipeline, the implementation only changes the scalar training label, the known intercept, and the loss used. In the reported experiments the intercept uses the population normalizer $\beta\log Z_{\lambda,\beta}$; Appendix~\ref{app:normalizer_incomplete_beta} also gives the exact finite-pool alternative $Z_{K,\lambda,\beta}$ for the empirical-rank convention used in Algorithm~\ref{alg:bce_alignment}.

\paragraph{Numerical evaluation of the population normalizer.}
Although Proposition~\ref{prop:optimal_policy_lambda} writes $Z_{\lambda,\beta}$ as an incomplete Beta function, the implementation evaluates the equivalent continuous integral directly with high-precision arithmetic. We do not use standard regularized incomplete-Beta routines such as \texttt{scipy.special.betainc}, since these APIs normalize by the complete Beta function and assume positive shape parameters, whereas here the second parameter is $1-1/\beta<0$ for the $\beta$ values used in our sweeps. We therefore compute
\begin{equation*}
Z_{\lambda,\beta}
=
\int_0^{1-\lambda}
u^{1/\beta}(1-u)^{-1/\beta}
\,du
\end{equation*}
using \texttt{mpmath} with \texttt{mp.workdps(256)}. This avoids the underflow that can occur in standard double-precision quadrature for small $\beta$ and makes the BCE intercept $\beta\log Z_{\lambda,\beta}$ reproducible.

\subsection{Hyperparameter search and model selection}

\begin{table}[t]
\centering
\caption{Training hyperparameters Llama 8B Tülu 3 SFT on UltraFeedback reported in Table~\ref{tab:llama_indataset_lc_reward}, Table~\ref{tab:llama_alpacaeval} and Table~\ref{tab:llama_ultrafeedback_raw_reward}.}
\label{tab:llama_ultrafeedback_hparams}
\resizebox{0.70\textwidth}{!}{%
\begin{tabular}{l l c c c c c}
\toprule
Method & Dataset variant & $\beta$ & $lr$ & Num ref & $\lambda$ & Checkpoint \\
\midrule
DPO & offpolicy2best & 0.01 & 3e-07 & 6 & - & 300 \\
REBEL & offpolicy2best & 1e-06 & 3e-07 & 6 & - & 200 \\
REBEL (rand.) & offpolicy2random & 1e-06 & 1e-06 & 6 & - & 100 \\
QRPO & offpolicy2best & 3e-04 & 3e-07 & 3 & - & 100 \\
QRPO (rand.) & offpolicy2random & 3e-04 & 3e-07 & 3 & - & 100 \\
BoNBoN & offpolicy2best & 1e-03 & 1e-06 & 6 & - & 476 \\
\midrule
TUP mild & offpolicy2random & 0.01 & 1e-07 & 6 & 0.2 & 400 \\
TUP mid. & offpolicy2random & 0.01 & 1e-07 & 6 & 0.5 & 400 \\
TUP aggressive & offpolicy2random & 0.01 & 1e-07 & 6 & 0.8 & 476 \\
\bottomrule
\end{tabular}}
\end{table}

\begin{table}[t]
\centering
\caption{Training hyperparameters Llama 8B Tülu 3 SFT on Magpie Air reported in Table~\ref{tab:llama_indataset_lc_reward}, Table~\ref{tab:llama_alpacaeval} and Table~\ref{tab:llama_magpie_raw_reward}.}
\label{tab:llama_magpie_hparams}
\resizebox{0.70\textwidth}{!}{%
\begin{tabular}{l l c c c c c}
\toprule
Method & Dataset variant & $\beta$ & $lr$ & Num ref & $\lambda$ & Checkpoint \\
\midrule
DPO & offpolicy2random & 0.01 & 1e-06 & 6 & - & 480 \\
REBEL & offpolicy2best & 1e-06 & 1e-07 & 6 & - & 640 \\
REBEL (rand.) & offpolicy2random & 1e-04 & 1e-06 & 6 & - & 320 \\
QRPO & offpolicy2best & 1e-03 & 1e-07 & 1 & - & 320 \\
QRPO (rand.) & offpolicy2random & 1e-03 & 1e-07 & 1 & - & 320 \\
BoNBoN & offpolicy2best & 1e-03 & 1e-06 & 6 & - & 764 \\
\midrule
TUP mild & offpolicy2random & 0.01 & 1e-07 & 6 & 0.2 & 320 \\
TUP mid. & offpolicy2random & 0.01 & 1e-07 & 6 & 0.5 & 320 \\
TUP aggressive & offpolicy2random & 0.01 & 1e-07 & 6 & 0.8 & 320 \\
\bottomrule
\end{tabular}}
\end{table}

\begin{table}[t]
\centering
\caption{Training hyperparameters Mistral 7B SFT on Magpie Air reported in Table~\ref{tab:mistral_sft_magpie_lc_reward} and Table~\ref{tab:mistral_sft_magpie_raw_reward}.}
\label{tab:mistral_sft_magpie_hparams}
\resizebox{0.7\textwidth}{!}{%
\begin{tabular}{l l c c c c c}
\toprule
Method & Dataset variant & $\beta$ & $lr$ & Num ref & $\lambda$ & Checkpoint \\
\midrule
DPO & offpolicy2random & 0.03 & 3e-07 & 6 & - & 480 \\
REBEL & offpolicy2best & 1e-04 & 1e-07 & 6 & - & 480 \\
REBEL (rand.) & offpolicy2random & 1e-04 & 3e-07 & 6 & - & 320 \\
QRPO & offpolicy2best & 3e-03 & 1e-07 & 1 & - & 764 \\
QRPO (rand.) & offpolicy2random & 3e-03 & 1e-06 & 1 & - & 764 \\
BoNBoN & offpolicy2best & 1e-03 & 1e-07 & 6 & - & 480 \\
\midrule
TUP mild & offpolicy2random & 0.01 & 1e-06 & 6 & 0.2 & 764 \\
TUP mid. & offpolicy2random & 0.01 & 1e-06 & 6 & 0.5 & 764 \\
TUP aggressive & offpolicy2random & 0.01 & 1e-06 & 6 & 0.8 & 764 \\
\bottomrule
\end{tabular}}
\end{table}

For BoNBoN, we use only the IPO-BoN component of the original objective. This gives the closest functional comparison in our setting, which is offline BoN-style distillation without an additional SFT loss. We tune only the learning rate:
\begin{equation*}
\text{lr}\in\{1\mathrm{e}{-6},3\mathrm{e}{-7},1\mathrm{e}{-7}\}.
\end{equation*}

For TUP using Llama model on UltraFeedback and Magpie Air, we tune the learning rate, preference sharpness, and truncation threshold over
\begin{equation*}
\text{lr}\in\{1\mathrm{e}{-6},3\mathrm{e}{-7},1\mathrm{e}{-7}\},
\quad
\beta\in\{3\mathrm{e}{-3}, 1\mathrm{e}{-2}, 3\mathrm{e}{-2}\},
\quad
\lambda\in\{0.2,0.5,0.8\}.
\end{equation*}

For Mistral 7B on Magpie Air, due to computational constraints, we ran the search on a narrower set of hyper parameters:
\begin{equation*}
\text{lr}\in\{1\mathrm{e}{-6},3\mathrm{e}{-7}\},
\quad
\beta\in\{1\mathrm{e}{-2}, 3\mathrm{e}{-2}\},
\quad
\lambda\in\{0.2,0.5,0.8\}.
\end{equation*}

Checkpoint selection uses only the validation split. We select checkpoints by length-controlled reward under ArmoRM. For both model families, UltraFeedback checkpoints are saved and evaluated every $100$ steps and Magpie Air checkpoints are saved and evaluated every $160$ steps. For UltraFeedback, we only consider checkpoints with mean generated length below $650$ tokens. For Magpie Air, we do not impose an analogous length threshold. Tables~\ref{tab:llama_ultrafeedback_hparams} and~\ref{tab:llama_magpie_hparams} report the selected hyperparameters and checkpoints for each method.

\subsection{Training settings}

Unless stated otherwise, training settings follow the QRPO benchmark configuration. Training uses \texttt{bfloat16} precision with Accelerate and DeepSpeed ZeRO-1. We use a cosine learning-rate schedule with warmup ratio $0.1$. Gradient clipping is effectively disabled in the Llama benchmark runs by setting \texttt{max\_grad\_norm} to $10^8$.

The per-device training batch size is $2$, the per-device evaluation batch size is $4$, and gradient accumulation is $16$. With $4$ GPUs, this gives an effective training batch size of $128$.

The maximum prompt length, maximum completion length, and maximum total sequence length are all $2048$. The released reference completions were generated with temperature $1.0$ and top-$p=1.0$. For online evaluation, we generate one completion per prompt using temperature $0.6$, top-$p=0.9$, and a maximum of $2048$ new tokens.

\subsection{Evaluation and uncertainty}
For reward-model evaluation, we follow the QRPO evaluation pipeline. For each generated completion, the evaluator computes both raw reward and length-controlled reward using the reference-completion pool for the same prompt. For UltraFeedback and Magpie Air, we use the released QRPO reference pools. For AlpacaEval, which is not part of the released QRPO reference data, we construct the reference pool by generating $16$ completions per prompt from the base model using vLLM. To compute length-controlled metrics under a given reward model, we score both the generated completions and the corresponding prompt-specific reference pools with that reward model. The length-controlled score normalizes reward and length relative to the prompt-specific reference pool, then removes the estimated linear effect of length. Prompts with degenerate reference-length variance are masked by the evaluation code.

For reward-model metrics, we report uncertainty as the sample standard deviation across three generation seeds. For AlpacaEval, we replace the recently \texttt{gpt-4-1106-preview} judge used in earlier benchmark configurations with \texttt{gpt-4o}. This change may make the AlpacaEval scores not directly comparable to results computed with the older judge.

\subsection{Deviation from the QRPO implementation}

The QRPO codebase constructs training examples using both chosen and rejected completions by default. For TUP, we instead use a single randomly selected response per prompt. This matches the scalar-label formulation of Algorithm~\ref{alg:bce_alignment}, where each completion receives one shifted-truncated win-rate label. It also means that TUP receives less per-prompt training signal than methods that use both responses.

\subsection{Compute resources}

The reported training runs were performed on a machine with $4$ H200 GPUs and $64$ CPU cores.

\subsection{Compliance and asset licenses}
\label{appx:asset_compliance}

Table~\ref{tab:asset_licenses} summarizes the external code, model, API, and dataset assets used in the experiments. License and terms information was retrieved from the official repository, model cards, dataset cards, and provider documentation on April 30, 2026. We use these assets only for offline training or evaluation as described in Appendix~\ref{appx:imp}, and we do not redistribute the underlying external model checkpoints, API-hosted judge, or source datasets as part of this submission.
\clearpage
\begin{table*}[thp]
\centering
\caption{External assets used in the experiments and their reported licenses or terms. For the QRPO reference data, the table lists the released preprocessed datasets actually loaded by our experiments; these are derived reference-completion pools and rewards rather than newly collected data in this work.}
\label{tab:asset_licenses}
\vspace{0.4em}
\footnotesize
\setlength{\tabcolsep}{4pt}
\renewcommand{\arraystretch}{1.16}

\begin{tabularx}{\textwidth}{@{}L{0.4\textwidth}Y L{0.26\textwidth}@{}}
\toprule
Asset & Role in this work & Reported license or terms \\
\midrule

\assetlink{https://github.com/CLAIRE-Labo/quantile-reward-policy-optimization}{QRPO reference codebase}
&
Experimental framework and baseline implementation.
&
MIT license.
\\

\rowrule

\href{https://huggingface.co/collections/skandermoalla/qrpo-reference-datasets}{QRPO UltraFeedback reference datasets}
\newline
{\scriptsize \texttt{offpolicy2best}, \texttt{offpolicy2random}}
&
Preprocessed UltraFeedback reference-completion pools and ArmoRM scores used for training and evaluation.
&
MIT license.
\\

\rowrule

\href{https://huggingface.co/collections/skandermoalla/qrpo-reference-datasets}{QRPO Magpie Air reference datasets}
\newline
{\scriptsize \texttt{offpolicy2best}, \texttt{offpolicy2random}}
&
Preprocessed Magpie Air reference-completion pools and ArmoRM scores used for training and evaluation.
&
MIT license for the QRPO preprocessed releases; the
\href{https://huggingface.co/datasets/Magpie-Align/Magpie-Air-DPO-100K-v0.1}{upstream Magpie Air dataset card}
does not declare a separate license field.
\\

\rowrule

\assetlink{https://huggingface.co/datasets/openbmb/UltraFeedback}{openbmb/UltraFeedback}
&
Source instruction-tuning dataset underlying the UltraFeedback QRPO setting.
&
MIT license.
\\

\rowrule

\assetlink{https://huggingface.co/datasets/tatsu-lab/alpaca_eval}{tatsu-lab/alpaca\_eval}
&
AlpacaEval prompt set used for GPT-judge evaluation.
&
CC-BY-NC-4.0 license.
\\

\rowrule

\assetlink{https://huggingface.co/spaces/allenai/reward-bench}{allenai/reward-bench}
&
RewardBench leaderboard context for the reward and judge models in Table~\ref{tab:rewardbench}.
&
ODC-BY-1.0 license.
\\

\rowrule

\assetlink{https://huggingface.co/allenai/Llama-3.1-Tulu-3-8B-SFT}{allenai/Llama-3.1-Tulu-3-8B-SFT}
&
Initial policy for the reported Llama benchmark runs.
&
Llama 3.1 Community License.
\\

\rowrule

\assetlink{https://huggingface.co/mistralai/Mistral-7B-Instruct-v0.2}{mistralai/Mistral-7B-Instruct-v0.2}
&
Initial policy for the reported Mistral benchmark runs (SFT on this model was ran separately).
&
Mistral Community License.
\\

\rowrule

\assetlink{https://huggingface.co/RLHFlow/ArmoRM-Llama3-8B-v0.1}{RLHFlow/ArmoRM-Llama3-8B-v0.1}
&
Offline proxy reward model used in the released QRPO labels and for checkpoint selection.
&
Llama 3 Community License.
\\

\rowrule

\assetlink{https://huggingface.co/Skywork/Skywork-Reward-V2-Llama-3.1-8B}{Skywork/Skywork-Reward-V2-Llama-3.1-8B}
&
Independent reward model used for transfer evaluation.
&
Llama 3.1 Community License.
\\

\rowrule

\assetlink{https://huggingface.co/Skywork/Skywork-Reward-V2-Qwen3-8B}{Skywork/Skywork-Reward-V2-Qwen3-8B}
&
Independent reward model used for transfer evaluation.
&
Apache-2.0 license.
\\

\rowrule

\href{https://developers.openai.com/api/docs/models/gpt-4o}{\texttt{gpt-4o}}
&
API-hosted judge for AlpacaEval.
&
Proprietary OpenAI API model used under
\href{https://openai.com/policies/service-terms/}{OpenAI service terms}; no weights are redistributed.
\\

\bottomrule
\end{tabularx}
\end{table*}
\clearpage
\newpage

\end{document}